\documentclass[11pt]{article}
\usepackage{graphicx}
\usepackage[margin=1in]{geometry}

\usepackage{algorithm}
\usepackage{float}
\usepackage{caption}
\usepackage{subcaption}
\usepackage{multirow, multicol}

\usepackage{amsmath, amssymb, amsthm, mathtools,stmaryrd, amssymb}

\usepackage[english]{babel}

\usepackage{enumitem}
\usepackage{lmodern, float}
\usepackage{sidecap}

\usepackage[utf8]{inputenc} 
\usepackage[T1]{fontenc}    
\usepackage{hyperref}       
\usepackage{url}            
\usepackage{amsfonts}       
\usepackage{nicefrac}       
\usepackage{microtype}      
\usepackage{braket}

\usepackage{fullpage}
\usepackage{natbib}
\usepackage{hyperref}
\usepackage{algpseudocode}
\newenvironment{proofof}[1]{\bigskip \noindent {\it Proof of #1.}\quad }
{\qed\par\vskip 4mm\par}

\newtheorem{theorem}{Theorem}[section]
\newtheorem*{theorem*}{Theorem}

\newtheorem{proposition}[theorem]{Proposition}
\newtheorem*{proposition*}{Proposition}
\newtheorem{lemma}[theorem]{Lemma}
\newtheorem*{lemma*}{Lemma}
\newtheorem{corollary}[theorem]{Corollary}
\newtheorem*{conjecture*}{Conjecture}
\newtheorem{fact}[theorem]{Fact}
\newtheorem*{fact*}{Fact}

\newtheorem*{hypothesis*}{Hypothesis}

\newtheorem{claim}[theorem]{Claim}
\newtheorem*{claim*}{Claim}

\theoremstyle{definition}

\theoremstyle{remark}
\newtheorem{remark}[theorem]{Remark}
\newtheorem*{remark*}{Remark}

\newtheorem*{observation*}{Observation}

\newcommand{\R}{\mathbb{R}}
\newcommand{\N}{\mathbb{N}}

\newcommand{\calN}{\mathcal{N}}

\newcommand{\poly}{\mathrm{poly}}

\newcommand{\norm}[1]{\lVert #1 \rVert}

\newcommand{\Bignorm}[1]{\Big\lVert#1\Big\rVert}

\newcommand{\iprod}[1]{\langle#1\rangle}

\newcommand{\bigiprod}[1]{\big\langle#1\big\rangle}
\newcommand{\Bigiprod}[1]{\Big\langle#1\Big\rangle}

\newcommand{\Esymb}{\mathbb{E}}
\newcommand{\Psymb}{\mathbb{P}}

\DeclareMathOperator*{\E}{\Esymb}

\DeclareMathOperator*{\ProbOp}{\Psymb}

\renewcommand{\Pr}{\ProbOp}

\newcommand{\tr}{\text{tr}}

\newcommand{\eps}{\varepsilon}
\renewcommand{\epsilon}{\varepsilon}

\newcommand{\sm}{\textsc{bot}}
\newcommand{\bg}{\textsc{top}}
\newcommand{\diag}{\textrm{diag}}

\newcommand{\wh}{\widehat}
\newcommand{\wt}{\widetilde}

\newif\ifnotes\notesfalse

\ifnotes
\usepackage{color}
\definecolor{mygrey}{gray}{0.50}
\newcommand{\notename}[2]{{\textcolor{blue}{\footnotesize{\bf (#1:} {#2}{\bf ) }}}}

\else

\newcommand{\notename}[2]{{}}

\fi

\newcommand{\pnote}[1]{{\notename{Pranjal}{#1}}}
\newcommand{\anote}[1]{{\notename{Aravindan}{#1}}}
\newcommand{\xnote}[1]{{\notename{Xue}{#1}}}

\title{Estimating Principal Components under Adversarial Perturbations}
\author{Pranjal Awasthi\\ \small{Google and Rutgers University}\\ \small{pranjalawasthi@google.com}  \and Xue Chen\\ \small{Northwestern University}\\ \small{xue.chen1@northwestern.edu} \and Aravindan Vijayaraghavan\thanks{The last author is supported by the National Science Foundation (NSF) under Grant No.~CCF-1652491 and CCF-1637585. }
 \\ \small{Northwestern University} \\ \small{aravindv@northwestern.edu}}
\date{}

\begin{document}

\maketitle
\begin{abstract}
    Robustness is a key requirement for widespread deployment of machine learning algorithms, and has received much attention in both statistics and computer science. 
    We study a natural model of robustness for high-dimensional statistical estimation problems that we call the {\em adversarial perturbation model}. An adversary can perturb {\em every} sample arbitrarily up to a specified magnitude $\delta$ measured in some $\ell_q$ norm, say $\ell_\infty$. Our model is motivated by emerging paradigms such as {\em low precision machine learning} and {\em adversarial training}. 
    
    We study the classical problem of estimating the top-$r$ principal subspace of the Gaussian covariance matrix in high dimensions, under the adversarial perturbation model. We design a computationally efficient algorithm that given corrupted data, recovers an estimate of the top-$r$ principal subspace with error that depends on a robustness parameter $\kappa$ that we identify. 
  This parameter corresponds to the $q \to 2$ operator norm of the projector onto the principal subspace, and generalizes well-studied analytic notions of sparsity. Additionally, in the absence of corruptions, our algorithmic guarantees recover existing bounds for problems such as sparse PCA and its higher rank analogs. 
    We also prove that the above dependence on the parameter $\kappa$ is almost optimal asymptotically, not just in a minimax sense, but remarkably for {\em every} instance of the problem. This  {\em instance-optimal} guarantee shows that the $q \to 2$ operator norm of the subspace essentially {\em characterizes} the estimation error under adversarial perturbations. \anote{Added remarkably, and removed suprising guarantee}     
\anote{Should we mention anything about mean estimation in passing?}
\end{abstract}

\section{Introduction}
\anote{1/30: Need to check that Gaussian is represented consistently. I sometimes write $N(0,1)$ instead of $\calN(0,1)$ etc.}
\anote{Should we call it Principal Subspace Estimation instead of covariance estimation?}
\anote{Mention why the convex program for mean estimation is polytime solvable. }

An important and active area of research in machine learning is the design of algorithms that are robust to modeling errors, noise and adversarial corruptions of different kinds. 
There is a rich body of work in the field of statistics, machine learning and theoretical computer science studying different models of robustness and the associated tradeoffs~\citep[e.g.][]{huber2011robust,tukey1975mathematics, hampel1986robust, diakonikolas2019robust, lai2016agnostic}. In the context of statistical estimation problems the most widely studied model is Huber's $\epsilon$-contamination model~\citep{huber2011robust}. In Huber's model it is assumed that a small $\eps$ fraction of the data set is corrupted arbitrarily. The remaining portion of the dataset that is left uncorrupted is assumed to be generated from a structured distribution such as a Gaussian. Other notions of robustness that have been explored in unsupervised learning include distribution closeness of different kinds~\citep{gao2018robust} and different semi-random models~\citep{BS92, FK99, MMV12}. Please see Section~\ref{sec:related} for more detailed comparisons. 

However there are several existing and emerging scenarios, where the data corruptions are not captured by these existing models of robustness.  
In many practical settings every data point is likely to perturbed with some small amount of noise, arising from various complex sources of errors. The reliability and security of learning algorithms could also be compromised by small imperceptible perturbations to the samples that are adversarial in nature (data poisoning). 
Moreover, {\em adversarial training} has emerged as a popular training paradigm where at each stage, the given training set is corrupted 
by adding (imperceptible) adversarial perturbations~(typically measured in $\ell_\infty$ or $\ell_2$ norm)~\citep{madry2017towards}, before performing stochastic gradient descent updates. This is empirically known to lead to more robust algorithms and also has implications for fair classification~\citep{madras2018learning}. 

Data corruptions also arise naturally in popular emerging paradigms like {\em low-precision machine learning}~\citep{de2017understanding, de2018high}. 
Low precision computation gives substantial savings in time and energy costs by storing and processing only a few most significant bits e.g., 8-bit arithmetic is a popular choice
. 
The lower memory utilization from low precision allows for processing of more training examples at the cost of quantization noise.
This quantization noise is naturally captured as a small adversarial perturbation to every co-ordinate of the data point to an amount that 
depends on the number of bits used in the arithmetic (an $\ell_\infty$ norm bound). 
These adversarial perturbations lead to new tradeoffs in the estimation accuracy that are not well understood for many basic statistical tasks.
In this work we take a step in this direction by studying a model of adversarial perturbations aimed at capturing the above scenarios.

\anote{Should we mention something about semi-random models?}

\anote{Para 3:Lead up paragraph for the model we study. We need to specifically make $\ell_\infty$ perturbations very natural. }

 \anote{Need more lines to motivate this.}

\paragraph{Adversarial Perturbation model.} We consider a natural model of robustness where {\em every} sample can be perturbed adversarially up to a bounded amount $\delta$, say in $\ell_\infty$ norm (more generally, in $\ell_q$ norm where $q \in (2, \infty]$ ). \anote{Should we just make it $q \in (2,\infty)$?} In our model the input data $\tilde{A} \in \R^{m \times n}$ consisting of $m$ samples in $\R^n$ is generated as follows: 

\begin{enumerate}
    \item The uncorrupted samples $A_1, \dots, A_m \in \R^n$ are drawn i.i.d. from a Gaussian $\calN(\mu, \Sigma)$, with unknown mean $\mu \in \R^n$ and $\Sigma  \in \R^{n \times n}$. 
    \item An adversary can observe the samples $A_1, \dots, A_m$, and perturb them arbitrarily to form $\tilde{A}_1, \dots, \tilde{A}_m \in \R^n$ such that for each $j \in [m]$,  $\norm{\tilde{A}_j - A_j}_q \le \delta$. These adversarial perturbations can be correlated. 
\end{enumerate}

We study the classical unsupervised learning problem of estimating the top-$r$ principal subspace of the covariance matrix $\Sigma$, and the best rank-$r$ approximation to $\Sigma$, for a specified $r \in [n]$. 
For $r=1$, this corresponds to recovering the principal component of $\Sigma$.

In the above model, the adversarial perturbations are measured in $\ell_q$ norm where $q \in (2, \infty]$. As $q$ goes to $\infty$, the perturbations become larger in magnitude and less constrained. When $q=\infty$, {\em every co-ordinate} of {\em every point} can get perturbed adversarially up to $\delta$ in magnitude. For the sake of exposition, we will focus on the case of $q = \infty$ and present results for general $q \in (2,\infty]$ in the respective sections.  

Our algorithms and guarantees will depend on certain quantity that we will call the {\em robustness parameter} $\kappa$, which captures the $q \to 2$ operator norm of the projector on to the target rank-$r$ subspace, and generalizes analytic notions of sparsity.  
Surprisingly, we will see that this robustness parameter will be crucial in characterizing the estimation error under our model. 
To understand why sparsity (and the $\infty \to 2$ operator norm) is related to robustness under adversarial perturbations, let us first consider the simpler setting of mean estimation. 

\paragraph{Warm up: Mean Estimation.}
Consider the problem of mean estimation where the uncorrupted data in $\R^n$ is generated from $\calN(\mu, I)$. A valid $\ell_\infty$ adversarial perturbation is moving each of the samples by the same vector $z=\delta(1, 1, \dots, 1)$, thereby moving the mean to $\mu'$ with $\norm{\mu' - \mu}_2^2 = \delta^2 n$. In this case no estimator can tell apart $\mu, \mu'$ from the data, hence this error of $\delta^2 n$ is unavoidable in the worst-case. Suppose however that mean $\mu$ was $k$-sparse 
i.e., it is supported on the set $S$ of size at most $k \ll n$. If the support $S$ is known beforehand, then by taking the empirical mean after projecting all the samples onto the support $S$, we can find an estimate $\widehat{\mu}$ with $\norm{\widehat{\mu}-\mu}_2^2 \le \delta^2 k \ll \delta^2 n$ asymptotically (as the number of samples goes to infinity). While we do not know the the sparse support of $\mu$ beforehand\footnote{This estimation problem is interesting even in the absence of adversarial perturbations, and corresponds to the {\em sparse mean estimation} problem that has been studied extensively in high-dimensional statistics~\cite{johnstone1994minimax, donoho1992maximum, donoho1994minimax}.}.  
the following proposition shows that one can indeed achieve the above improved rate when the mean is sparse in an analytic sense (the ratio of norms $\ell_1/\ell_2$ ). 
\anote{Needs to update the bound}
\begin{proposition}[Mean Estimation under Adversarial Perturbations]
\label{prop:intro:robust_mean}
Suppose we have $m$ samples drawn according to the Adversarial Perturbation model with mean $\mu$, covariance $\Sigma \preceq \sigma^2 I$ and $q=\infty$.  There is a polynomial time algorithm (Algorithm~\ref{algo:mean}) that outputs an estimate $\hat{\mu}$ for the (unknown) mean $\mu$ such that with probability at least $(1-1/n)$, 
\begin{align}
        \norm{\hat{\mu} - \mu}_2^2 &\leq 4\min\Big\{\norm{\mu}_1 ( \delta+ \eta),  n ( \delta+ \eta )^2 \Big\}, \text{ where } \eta:=2\sigma\sqrt{(\log n)/m}.
\end{align}
\end{proposition}

See Proposition~\ref{thm:robust_mean} for general statement for all $\ell_q$ norms. 
\xnote{May 31st: introduce $\kappa=\|\mu\|_1/\|\mu\|_2$ here.}If we use $\kappa=\frac{\|\mu\|_1}{\|\mu\|_2}$ to denote the analytic sparsity of $\mu$, the first error term becomes $\kappa \cdot (\delta +\eta) \cdot \|\mu\|_2$.
In fact, the above error of $\Omega(\kappa \delta \norm{\mu}_2)$ is unavoidable for {\em every} instance for a broad range of parameters i.e., for every instance of the problem, there exists an adversarial perturbation that makes it statistically impossible to recover the mean with error $o(\delta \norm{\mu}_1)$ (see Proposition~\ref{thm:state_lower_mean}).

\anote{Some motivation for sparsity? Should we include work from our previous work?}

\paragraph{Robustness Parameter $\kappa$.}
Similarly the estimation rates for finding the top-$r$ principal subspace (or best rank-$r$ approximation) of $\Sigma$ will be characterized by the robustness parameter $\kappa$ that is given by the $\infty \to 2$ operator norm:
 $$\norm{\Pi}_{\infty \to 2} = \max_{y : \norm{y}_\infty \le 1 } \norm{\Pi y}_2,$$      
 where $\Pi$ is the (orthogonal) projection matrix onto the subspace spanned by the top-$r$ eigenvectors of $\Sigma$ (for general $q$, the robustness parameter will correspond to $\norm{\Pi}_{q \to 2}$ operator norm). This robustness parameter generalizes analytic notions of sparsity (the ratio of $\ell_1/\ell_2$ norms) to projection matrices of subspaces\footnote{For the special case of a $1$-dimensional subspace along the vector $v$, the orthogonal projector $\Pi_1 = \tfrac{1}{\norm{v}_2^2} v v^\top$ satisfies $\norm{\Pi}_{\infty \to 2} = \norm{\Pi}_{2 \to 1}= \norm{v}_1/ \norm{v}_2$. See Fact~\ref{fact:operator_norm_pro} for details. }. Note that $\kappa$ takes values in $[1,\sqrt{n}]$. 
The $\infty\to 2$ operator norm is also related to the famous Grothendieck inequality from functional analysis~\citep{Gro56,alon2004approximating}. These parameters have also been used recently to characterize robustness to adversarial perturbations at test-time~\citep{ACCV} (see Section~\ref{sec:related} for more discussion). Similar to mean estimation, the case of $r=1$ for covariance estimation corresponds to the well studied sparse PCA problem~\citep{johnstone2001distribution, aminiwainwright, ma2013sparse,  VuLei12, VuLei13, BerthetRigollet}. Extensions of sparse PCA to estimating top $r$ ``sparse'' subspaces have also been widely studied in the statistics community~\citep{VuLei13, Liuetal} .

As we will see soon, our guarantees are not only minimax optimal in terms of these parameters, but they are essentially {\em instance-optimal}! Our upper bound and lower bound guarantees will work for {\em every} instance and will be tight up to logarithmic factors asymptotically (as number of samples becomes large). 
Hence our results give a surprising characterization of the estimation error under adversarial perturbations in terms of these robustness parameters (measured in $\infty \to 2$ norm), and highlight new robustness benefits of sparsity in high dimensional estimation. 

\anote{Emphasize how this highlights an additional benefit of sparsity. Maybe move this earlier?}

\subsection{Our Results}\label{sec:results}

\anote{Explain different settings to illustrate instance-optimality }

We now state our main results on recovering the principal subspace (and the best rank-$r$ approximation) of the covariance $\Sigma^*$ in terms of the $\infty \to 2$ operator norm of the corresponding rank-$r$ projection matrix. 
The samples are drawn from the Adversarial Perturbation model where the covariance of the uncorrupted samples $\Sigma^*$ has eigenvalues $\lambda_1 \ge \lambda_2 \ge \dots \ge \lambda_n \ge 0$. 
The unknown covariance matrix is split into $\Sigma = \Sigma_\bg+\Sigma_\sm$, where $\Sigma_\bg$ corresponds to the best rank-$r$ approximation of $\Sigma$ i.e., the truncation of the SVD to the top-$r$ eigenvalues $\lambda_1, \dots, \lambda_r$. Let $\Pi^*$ be the orthogonal projection matrix onto the span of $\Sigma_\bg$. We will assume that $\norm{\Pi^*}_{\infty \to 2}\le \kappa$. We will measure the estimation error in squared Frobenius norm. For the case of projection matrices, this is equivalent (up to a factor of $2$) to the standard notion of subspace $\sin\Theta$ distance~(see Section~\ref{sec:prelims_app}).
\anote{Describe $\sin \Theta$ error?}

\begin{theorem}\label{ithm:comp_upper_gen}[Algorithm]
Suppose we have $m$ samples drawn according to the the above Adversarial Perturbation model with (unknown) covariance $\Sigma^*$ satisfying $\norm{\Pi^*}_{\infty \to 2} \le \kappa$. Assuming that $\kappa \delta \le \frac{O(\lambda_r-\lambda_{r+1})}{\sqrt{r \lambda_1}}$, there exists an algorithm (Algorithm~\ref{algo:covariance}) that for any $\eps>0$ uses $m \geq C r^2 \kappa^4 \big(\tfrac{\lambda_1^2}{(\lambda_r-\lambda_{r+1})^2}\big) \log n/\eps^2  $ samples and outputs a rank-$r$ projection $\widehat{\Pi}$ with $\norm{\widehat{\Pi}}_{\infty \to 2} = O(\kappa)$, and an estimate $\widehat{\Sigma}_\bg$ (restricted to the subspace $\widehat{\Pi}$) such that
\begin{align*}
\|\widehat{\Pi} -  \Pi^*\|_F^2& \le \eps_1 := \tfrac{\sqrt{\lambda_1}}{(\lambda_r-\lambda_{r+1})} \cdot O\big( \sqrt{ r} \cdot \kappa \delta \big) +\eps \text{ and }
 \|\widehat{\Sigma}_{\bg}-\Sigma_{\bg}\|_F^2 \le O(\lambda_1^2 \eps_1 + \lambda_1 \kappa^2 \delta^2).
\end{align*}



\end{theorem}
See Theorem~\ref{thm:comp_upper_gen} for the general statement for $q > 2$ and the proof.
To interpret the results let's consider the case when $\Sigma^*= \theta \Pi^*+I$ (hence $\Sigma_\bg =(1+\theta) \Pi^* $), and $\theta = \Theta(1)$.\footnote{When $r=1$, this special case is the sparse PCA setting where the principal component has $\ell_1$ sparsity $\kappa$.} The above theorem shows that there is an efficient algorithm that obtains a rank-$r$ projection $\widehat{\Pi}$ that is $O(\sqrt{r} \kappa \delta)$ close to $\Pi^*$ in squared Frobenius norm, for sufficiently large polynomial $m$ ($\widehat{\Pi}$ also has robustness parameter $O(\kappa)$). On the other hand, a random subspace of rank $r$ will incur an error of $\Omega(r)$. Our algorithm can achieve an error of $o(1)$ while tolerating an additive perturbation that is as large as  $\delta = o( 1/(\sqrt{r}\kappa) )$ (which could be $n^{-0.21}/\sqrt{r}$ if $\kappa=n^{0.2}$).
On the other hand, if the top-$r$ subspace has no special structure (robustness parameter $\kappa \approx \sqrt{n}$), then one requires $\delta =o(n^{-1/2}/\sqrt{r})$ for achieving similar error rates.
Next, we give a computational inefficient algorithm that achieves a better statistical rate in terms of the sample complexity. 

\begin{theorem}\label{ithm:info_upper_gen}[Statistical upper bound]
Given $m$ samples drawn according to the Adversarial Perturbation model with covariance $\Sigma^*$ satisfying $\norm{\Pi^*}_{\infty \to 2} \le \kappa$, there exists an algorithm that for any $\eps>0$ uses $m \geq C r^2 \kappa^2 \big(\tfrac{\lambda_1^2}{(\lambda_r-\lambda_{r+1})^2}\big) \log n/\eps^2  $ samples and outputs a rank-$r$ projection $\widehat{\Pi}$ with $\norm{\widehat{\Pi}}_{\infty \to 2} \le \kappa$, and an estimate $\widehat{\Sigma}_\bg$ (restricted to the subspace $\widehat{\Pi}$) s.t.
\begin{align}
\|\widehat{\Pi} -  \Pi^*\|_F^2& \le \eps_1 := \tfrac{\sqrt{\lambda_1}}{(\lambda_r-\lambda_{r+1})} \cdot O\big( \sqrt{ r} \cdot \kappa \delta \big) +\eps \text{ and }
 \|\widehat{\Sigma}_{\bg}-\Sigma_{\bg}\|_F^2 \le O(\lambda_1^2 \eps_1 + \lambda_1 \kappa^2 \delta^2).\nonumber
\end{align}
\end{theorem}
See Theorem~\ref{thm:inf_recover_upper_gen} for the guarantees for general $q>2$. 
The dominant error of $O(\sqrt{ r}\kappa \delta)$ is the same for both Theorems~\ref{ithm:comp_upper_gen} and~\ref{ithm:info_upper_gen}, and represents the asymptotic error (error as $m \to \infty$). The main difference however is the number of samples $m$ needed as a function of $\kappa$ to drive the error to within $\eps$ of this asymptotic error. This gap of $\kappa^4$ vs $\kappa^2$ represents a computational vs statistical tradeoff that is unavoidable even when $r=1$ (and $q=\infty$), assuming the hardness of the Planted Clique problem. This follows directly from computational lower bounds for sparse PCA with a $k=\kappa^2$-sparse vector (combinatorial sparsity) assuming Planted Clique hardness~\citep{BerthetRigollet,gao2017sparse}. For smaller $q \in (2,\infty)$, there is an extra polynomial factor gap of $n^{2/q}$ in the sample complexity between Theorem~\ref{thm:comp_upper_gen} and Theorem~\ref{thm:inf_recover_upper_gen} that would be interesting to resolve. Finally 
the estimation error in the absence of any adversarial errors is comparable to the existing state of the art results that are known to be tight (minimax optimal)~\citep{VuLei13,ACCV}.

The following lower bound shows that our asymptotic error guarantees are almost optimal for {\em every} instance. 

\begin{theorem}\label{ithm:inf_recover_lower}[Lower Bound] \label{thm:inf_recover_lower}
Suppose we are given parameters $r\in \N, \kappa \ge 2r$ and $\delta>0$. 
In the notation of Theorem~\ref{ithm:info_upper_gen}, for any $\Sigma^*$, given $m$ samples $A_1, \dots, A_m$ generated i.i.d. from $\calN(0,\Sigma^*)$ with 
$\kappa=\norm{\Pi^*}_{\infty \to 2}$ satisfying $\sqrt{r \lambda_1} (\kappa/n) \le \delta \le \sqrt{r \lambda_1}/\kappa$, 
there exists a covariance matrix $\Sigma'$ with a projector $\Pi'$ onto its top-$r$ principal subspace, and an alternate 
dataset $A'_1, \dots, A'_m$ drawn i.i.d. from $\calN(0,\Sigma')$ satisfying   $\norm{\Pi'}_{\infty \to 2} \le (1+o(1))\kappa$, and $\norm{A'_j - A_j}_\infty \le \delta ~\forall j \in [m]$,  
\begin{align*}
     \text{but }\norm{\Pi^*-\Pi'}_F^2 &\ge \big(\tfrac{\Omega(1)}{\sqrt{\lambda_1} \log(rm) \log n}\big) \cdot \sqrt{r }\kappa \delta,  ~\text{and } & \norm{\Sigma'_\bg - \Sigma_\bg}_F^2 \ge \tfrac{(\lambda_1^2+\dots+\lambda_r^2)}{r} \cdot \norm{\Pi'-\Pi^*}_F^2
\end{align*}
 In particular, when $\Sigma_\bg=(1+\theta) \Pi^*$ then $\Sigma'_\bg=(1+\theta') \Pi'$ with $\theta'=(1+o(1))\theta$. 

\end{theorem}
See Section~\ref{sec:lower_bound} for more details and proof of the construction, and Theorem~\ref{thm:inf_recover_lower_q} for the extension to general $\ell_q$ norms.  
Consider the previous setting where $\lambda_r - \lambda_{r+1} =\Omega(\lambda_1)$ and think of $m$ as being any large polynomial in $n$.
The above lower bound on the error $\norm{\Pi' -\Pi^*}_F^2 = \tilde{\Omega}(\sqrt{r} \kappa \delta)$ nearly matches the error bound obtain by our algorithm in Theorem~\ref{ithm:comp_upper_gen} (as $m$ becomes a sufficiently large polynomial and hence $\eps \approx 0$) up to logarithmic factors, for {\em every} instance (i.e., every $\Pi^*,\Sigma^*$) i.e., our bounds are nearly {\em instance-optimal}. Note that this is much stronger than minimax optimality, which only requires the lower bounds to be tight for a specific choice of $\Sigma^*,\Pi^*$.      
Hence, Theorem~\ref{ithm:comp_upper_gen} and Theorem~\ref{ithm:inf_recover_lower}
together show that the $\infty \to 2$ norm of the projection matrix essentially {\em characterizes} the robustness to training errors bounded in $\ell_\infty$ norm.  

\paragraph{Discussion of the characterization.} Our characterization of the robustness to adversarial perturbations is in terms of the robustness parameter $\kappa=\norm{\Pi^*}_{\infty\to 2}$ ($\norm{\Pi^*}_{q \to 2}$ for general $q$), which generalizes analytic notions of sparsity. For a $r=1$-dimensional subspace, this exactly corresponds to the $\ell_1$ sparsity of the unit vector $v$ in that subspace. For higher-dimensional subspaces, there are several other notions of sparsity that have been explored~\citep{VuLei13, Liuetal}. For a fixed orthonormal basis $V \in \R^{n \times r}$ of the subspace (so $\Pi^*=VV^\top$), some of the notions that have been considered include the entry-wise norm $\norm{V}_1$ (the sum of the $\ell_1$ norms of the basis vectors), the maximum $\ell_1$ norm among the columns of $V$, the sparsity of the diagonal of $\Pi^*$ and the sum of the row $\ell_2$ norms of $V$, among other quantities. Many of these quantities are the same for $r=1$ but may vary by factors of $\sqrt{r}$ or more depending on the quantity. On the other hand, our robustness parameter $\kappa$ is a property only of the subspace and is basis independent. 
The $\norm{\Pi^*}_{\infty \to 2}$ of a projector is the largest $\ell_1$ norm among unit vectors (in $\ell_2$ norm) that belong to the subspace.

Consider three different subspaces (or projectors) given by the orthonormal basis $V_1, V_2, V_3 \in \R^{n \times r}$ of the following form (think of $\kappa = \sqrt{k}$, $r \ll \kappa$); assume that the signs of the entries are chosen randomly in a way that also satisfies the necessary orthogonality properties (e.g., random Fourier characters over $\set{\pm 1}^k$).
\begin{equation*}
V_1= 
\begin{pmatrix}
 \tfrac{\pm 1}{\sqrt{k}} &  \tfrac{\pm 1}{\sqrt{k}} & \cdots &  \tfrac{\pm 1}{\sqrt{k}} \\
 \tfrac{\pm 1}{\sqrt{k}} &  \tfrac{\pm 1}{\sqrt{k}} & \cdots &  \tfrac{\pm 1}{\sqrt{k}} \\
\vdots  & \vdots  & \ddots & \vdots  \\
 \frac{\pm 1}{\sqrt{k}} &  \frac{\pm 1}{\sqrt{k}} & \cdots &  \frac{\pm 1}{\sqrt{k}} \\
 0 & 0 & \cdots & 0 \\
\vdots & \vdots & \vdots & \vdots \\
0 & 0 & \cdots & 0 
\end{pmatrix},
~~~~
V_2= 
\begin{pmatrix}
\tfrac{\pm \sqrt{r}}{\sqrt{k}} & 0 & \cdots & 0 \\
\cdot & \cdot & \cdots & \cdot \\
\tfrac{\pm \sqrt{r}}{\sqrt{k}} & 0 & \cdots & 0 \\
0 &  \tfrac{\pm \sqrt{r}}{\sqrt{k}} & \cdots & 0 \\
\cdot & \cdot & \cdots & \cdot \\
0 & \tfrac{\pm \sqrt{r}}{\sqrt{k}} & \cdots & 0 \\
0 &  0 & \cdots & 0 \\
\vdots  & \vdots  & \ddots & \vdots  
\end{pmatrix},
~~~V_3= 
\begin{pmatrix}
\tfrac{\pm 1}{\sqrt{r}} & \tfrac{\pm 1}{\sqrt{r}} & \cdots & \tfrac{\pm 1}{\sqrt{r}}& \tfrac{\pm 1}{\sqrt{k}} \\
\cdot & \cdot & \cdots &\cdot & \cdot \\
\tfrac{\pm 1}{\sqrt{r}} & \tfrac{\pm 1}{\sqrt{r}} & \cdots & \tfrac{\pm 1}{\sqrt{r}}& \tfrac{\pm 1}{\sqrt{k}} \\
0 & 0 & \cdots&0 & \tfrac{\pm 1}{\sqrt{k}} \\
\cdot & \cdot & \cdots&\cdot & \cdot \\
0 & 0 & \cdots & 0& \tfrac{\pm 1}{\sqrt{k}} \\
0 &0 &\cdots & 0&0\\
\vdots  & \vdots  & \ddots & \vdots & \vdots  
\end{pmatrix}
\end{equation*} 
The main difference between $V_1, V_2$ is that in $V_2$ the sparse basis vectors have disjoint support, whereas in $V_1$ they are commonly supported. However, there is an alternate basis for the subspace $V_2$ which looks like $V_1$, but basis dependent quantities like the maximum $\ell_1$ norm among columns get very different values for $V_1, V_2$. In the third example, the first $r-1$ basis vectors are extremely sparse with $\ell_1$ norm $O(\sqrt{r})$, whereas only one of the basis vectors has $\ell_1$ sparsity $\sqrt{k}$. Many aggregate notions of sparsity like $\norm{V}_1$  or sum of the row $\ell_2$ norms have very different values for $V_1$ and $V_3$ that differ by a $\sqrt{r}$ factor. On the other hand, our robustness parameter $\kappa \approx \sqrt{k}$; this is because each of these subspaces are supported on at most $k$ co-ordinates (and a spread out vector of this form exists), so the maximum $\ell_1$ length among unit $\ell_2$ norm vector is $\sqrt{k}$. Hence, while our robustness parameter $\norm{\Pi^*}_{\infty \to 2}$ {\em characterizes} the asymptotic error that can be obtained in all of these different cases (using Theorem~\ref{ithm:comp_upper_gen} and Theorem~\ref{ithm:inf_recover_lower}), many other natural notions of sparsity are off by factors of $\sqrt{r}$ or more in at least one of these cases.

Finally, our robustness parameter $\kappa$ also satisfies other useful properties like monotonicity (see Lemma~\ref{lem:monotone}), that will be very useful in the algorithm and analysis (this is not satisfied by various other norms like $\norm{\cdot}_1$ etc.). While the $\infty \to 2$ operator norm is NP-hard to compute for PSD matrices, there exists polynomial time algorithms that can compute it up to a small constant factor (that corresponds to the Gr\"{o}thendieck constant for PSD matrices)~\citep[see ][]{nesterov1998semidefinite, alon2004approximating}.

\paragraph{Comparison to Prior Work and Related Work.}
\label{sec:comparison}

There are several other notions of robustness that have been explored in both unsupervised and supervised learning. We place our work in the context of these existing works in Section~\ref{sec:related}. The work that is closest to this paper is the recent work of \citet{ACCV}. Our work is inspired by \citet{ACCV} and builds on some of those techniques. However, our work differs significantly from \cite{ACCV} both in terms of the problem focus, and the nature of the results, as we explain below. The main problem considered in \cite{ACCV} is finding a low-rank projection of a given data matrix $A$ that achieves low approximation error, and is also robust to adversarial perturbations at {\em testing} time. Robustness at test time naturally places an upper bound constraint on the $q \to 2$ operator norm of the projection matrix.
The paper also consider this problem under adversarial perturbations at training-time, and use these results as a black-box to obtain some guarantees for mean estimation and clustering in the presence of adversarial perturbations. The paper mainly studies the worst-case setting which is computationally hard, and hence focus on multiplicative approximation guarantees for an objective (like low-rank approximation error), as opposed to estimation or recovery.

On the other hand, the main focus of this paper is adversarial perturbations at {\em training time}; there is no requirement of robustness at testing-time. Hence, it is not clear why $\kappa=\norm{\Pi}_{q \to 2}$ is a relevant parameter at all. The main message of this paper is that this parameter $\kappa$ indeed characterizes the robustness to adversarial perturbations at training time as well (this is even if test-time robustness is not a consideration)! Moreover  we focus on high-dimensional statistical estimation tasks where there is an underlying distribution for the uncorrupted data, and allows us to obtain the strong statistically optimal recovery guarantees. Hence the guarantees in the two works are incomparable. 
\section{Preliminaries}\label{sec:prelims}
\paragraph{Norms.} For a vector $v \in \R^n$ and any $q \ge 1$, we use $\|v\|_{q}$ to denote the $q$-norm: $\big( \sum_{i=1}^n |v(i)|^q \big)^{1/q}$. For any fixed $q \ge 1$, we use $\ell_{q^*}$ to denote the dual of $\ell_q$, where $1/q+1/q^*=1$. We also apply H\"older's inequality extensively:
$\forall q \ge 1 \text{ and } u,v \in \R^n, \big| \langle u, v \rangle \big| \le \|u\|_{q^*} \|v\|_q.
$ A direct corollary is that $\|v\|_q \le |\text{support}(v)|^{1/q-1/p} \cdot \|v\|_p$ for any vector $v$ and any $q<p$. In particular, $\|v\|_1 \le \sqrt{k}$ for a unit vector $v$ of sparsity $k$.

For a matrix $A \in \R^{n \times m}$ and $q \ge 1$, we will use $\|A\|_q$ to denote the entry-wise $\ell_q$ norm of $A$: $\big(\sum_{i,j} |A(i,j)|^q \big)^{1/q}$. When $q=2$, we will also use the Frobenius norm $\|A\|_F \overset{\text{def}}{=} \|A\|_2$ equipped with trace inner product $\langle A, B \rangle=\tr(A^{\top} B)$.

\paragraph{$p \to q$ norms.} For any $p$ and $q$, we define the operator $p \to q$ norm for a matrix $A \in \R^{n \times m}$:
\[
\|A\|_{p \to q}=\max_{v \in \R^m \setminus \set{0}} \|A v\|_q/\|v\|_p.
\]

For convenience, let $\|A\|$ denote the operator norm $\|A\|_{2 \to 2}$. A variational definition of the operator norm is as follows (See Section 4 in \cite{ACCV} for proofs).
\begin{fact}\label{fact:dual_norms}
For any $p$ and $q$, 
$\|A\|_{p \to q} = \underset{u \in \R^n \setminus \set{0}, v \in \R^m \setminus \set{0}}{\max} {u^{\top} A v}/{(\|u\|_{q^*} \|v\|_p)}$. Also, $\|A\|_{p \to q}=\|A^{\top}\|_{q^* \to p^*}$ and $\|A^\top A\|_{q \to q^*} = \|A\|^2_{q \to 2}.$ In particular, $\|\Pi\|_{\infty \to 2}=\|\Pi\|_{2 \to 1}$ and $\|\Pi\|_{q \to q^*}=\|\Pi\|_{q \to 2}^2$ for projection matrices.
\end{fact}
Due to the space constraint, we defer a few properties of the operator norm to Appendix~\ref{sec:prelims_app}.

\section{Computational Upper Bound}\label{sec:comp_upper}
In this section we present our computationally efficient algorithm for estimating the top-$r$ principal subspace. We state our main claim regarding the error guarantees associated with the algorithm and describe the key ideas used in the analysis. All the proofs are deferred to Appendices \ref{sec:comp-upper-app} and \ref{sec:comp_upper_app}.
A key subroutine in our algorithm is the following convex program that was proposed in \cite{ACCV}. We use the program will be run on the corrupted data $\tilde{A}$ and the bulk of our analysis will involve showing that the solution output by the program can be used for estimation in spite of adversarial perturbations.
The program takes in as parameters the rank $r$ and an upper bound for the robustness parameter $\kappa$, whose target solution is the projection $\Pi^*$ of $\Sigma^*$. 
\begin{align}
& \min \frac{1}{m}\norm{A}_F^2 - \frac{1}{m} \iprod{A A^{\top}, X} \label{sdp:spike}\\
\textit{ subject to } ~& \tr(X)\le r\\
& 0 \preceq X \preceq I\\
& \norm{X}_{q^*} \le r \kappa^{2} \label{sdp:totalnorm}\\
&\norm{X}_{q \to q^*} \le \kappa^2 \label{sdp:robustnorm}
\end{align}
One can use the Ellipsoid algorithm to efficiently solve the program above via an efficient separation oracle~(See Lemma~\ref{lem:sdpsolving:spike}). We briefly discuss the last two constraints in the above program and refer to \cite{ACCV} for a more detailed discussion: The constraint~\eqref{sdp:totalnorm} comes from the fact that the projection $\Pi^*=\sum_{i=1}^r v_i v_i^{\top}$ where each $\|v_i\|_{q^*} \le k$. At the same time, the last constraint~\eqref{sdp:robustnorm} is based on the monotonicity of $q \rightarrow q^*$ norms from Lemma~\ref{lem:monotone}.

Below is the algorithm that uses the SDP solution above to outputs a robust projection matrix $\widehat{\Pi}$ of rank at most $r$. 

\begin{algorithm}[H]
\caption{Finding Robust Low-Rank Projection}
\begin{algorithmic}[1]
\Function{RobustProjection}{data matrix $\tilde{A} \in \R^{m \times n}$, rank $r$, robustness $\kappa$, norm $q$}
\State Solve \eqref{sdp:spike} on $\tilde{A}$ with parameters $\kappa, q, r$ to find a solution $\widehat{X} \succeq 0$ (see Lemma~\ref{lem:sdpsolving:spike}). 
\State Use SVD on $\widehat{X}$ to find the subspace spanned by the top-$r$ eigenvectors of $\widehat{X}$. Output $\widehat{\Pi}$, the orthogonal projection matrix onto this subspace. 
\EndFunction \label{algo:projection}
\end{algorithmic}
\end{algorithm}

Finally, our algorithm for estimating the principal components of the covariance matrix in the presence of adversarial perturbations, described below, just uses {\sc RobustProjection} as an additional pre-processing step to find a suitable robust subspace for computing the empirical covariance. 

\begin{algorithm}[H]
\caption{Principal Subspace Estimation under Adversarial Perturbations}
\label{algo:covariance}
\begin{algorithmic}[1]
\Function{AdvRobustPCA}{$4m$ samples $\tilde{A}_1, \dots, \tilde{A}_{4m} \in \R^{n}$, rank $r$, robustness $\kappa$, $q$}
\State Split samples into two equal parts. Let $A^{(1)}, A^{(2)}$ denote these two datasets.   
\State For each $j \in [m]$, let $A'_j = \tfrac{1}{\sqrt{2}}(\tilde{A}_j - \tilde{A}_{m+j})$ and let $A''_j = \tfrac{1}{\sqrt{2}}(\tilde{A}_{2m+j} - \tilde{A}_{3m+j})$. \label{st:sym}
\State Run \textsc{RobustProjection}$(A', r, \kappa,q)$ to find a $r$-dimensional projection matrix $\widehat{\Pi}$. 
\State Output $\widehat{\Sigma}_r$ to be empirical covariance of $\widehat{\Pi} A''$. 
\EndFunction
\end{algorithmic}
\end{algorithm}
Next, we state our main theorem regarding the estimation error associated with the algorithm above. We state the guarantee for a general $q \geq 2$. Substituting $q = \infty$ recovers the guarantee stated in Theorem~\ref{ithm:comp_upper_gen}.

\begin{theorem}\label{thm:comp_upper_gen}
Given $q \geq 2$, $r$, and $\kappa$, let $\wt{A} \in \mathbb{R}^{n \times m}$ be a $\delta$-perturbation (in $\ell_q$ norm) of data points generated from $\calN(0,\Sigma^*)$. Let $\lambda_1 \ge \lambda_2 \ge \dots \ge \lambda_n$ be the eigenvalues of the covariance matrix $\Sigma^*$ and $\Pi^*$ be the projection matrix on to the top $r$ eigenspace of $\Sigma^*$. There exists a universal constant $C$ such that for any $\eps>0$, and $\kappa \delta \leq \frac{\lambda_r-\lambda_{r+1}}{C   \sqrt{r  \lambda_1}}$, Algorithm~\ref{algo:covariance} when provided with $m \ge C r^2 \kappa^4 \cdot \frac{\lambda_1^2}{(\lambda_r-\lambda_{r+1})^2} \log n \cdot \frac{n^{4/q}}{\eps^2}$ samples, outputs with probability at least $0.99$ $\wt{\Sigma}_{\bg}$ of rank $r$ and the projector onto its subspace $\wt{\Pi}$ that satisfies $\norm{\wt{\Pi}}_{q \to 2} = O(\kappa)$,
$$\norm{\wt{\Pi} - \Pi^*}_F^2 \le O\Big(\frac{\sqrt{ \lambda_{1}  r} \cdot \kappa \delta}{\lambda_r-\lambda_{r+1}} \Big) + \eps \text{ and } \| \wt{\Sigma}_{\bg}-\Sigma_{\bg} \|_F^2 \le O\Big(\lambda_1^2  \|\wt{\Pi}- \Pi^*\|_F^2 + \lambda_1 \kappa^2 \delta^2\Big) .$$ 
\end{theorem}
Before the proof of Theorem~\ref{thm:comp_upper_gen}, We describe the key ideas and supporting claims that are used in our analysis.
The proof consists of three main steps. We first argue about the error of the estimated projection matrix $\wt{\Pi}$ with respect to $\Pi^*$. One can show that the optimal solution to the convex program~\eqref{sdp:spike} (that we will refer to as the SDP) on the ideal instance $\E[AA^\top]$ in fact recovers the projection $\Pi^*$. However the SDP is solved on the given instance $\E[AA^\top]+E$ where $E$ is the error matrix defined as $E := \frac{1}{m} \tilde{A}\tilde{A}^\top - \E[AA^\top]$ involving both the adversarial perturbations and sampling errors. The first part of the argument for the robustness of the SDP to adversarial perturbations is by providing an upper bound on $|\iprod{E,X}|$ over all feasible SDP solutions $X$.  Lemma~\ref{lem:sdp_recovery_gen} that is stated below  crucially uses the constraints on $\norm{X}_{q \to q^*}$ and $\norm{X}_{q^*}$ to provide the required bound.

\begin{lemma} \label{lem:sdp_recovery_gen}
Let $\tilde{A}$ be a $\delta$-perturbation (in $\ell_q$ norm) of the original data matrix $A$ where $\E[A A^\top]=\Sigma^*$. Let $
E:= \frac{1}{m} \tilde{A}\tilde{A}^\top - \E[AA^\top]$ denote the error matrix and define \[
\mathcal{P}_{c(q)} = \{X \in \R^{n \times n}: \tr(X)=r,
0 \preceq X \preceq I,
\norm{X}_{q^*} \le r \kappa^{2},
\norm{X}_{q \to q^*} \le c(q) \cdot \kappa^2\}\] 
as the set of all solutions that can be obtained by solving the SDP in \eqref{sdp:spike} via the Ellipsoid Algorithm~(see Lemma~\ref{lem:sdpsolving:spike}). 
With high probability, 
$\Delta:= \sup_{X \in \mathcal{P}_{c(q)}} |\iprod{E,X}|$ satisfies 
\[
\Delta \le  O\Big( \sqrt{r\cdot \lambda_{\max}(\Sigma^*)} \kappa \delta + \kappa^2 \delta^2+ \frac{r \kappa^2 \cdot \lambda_{\max}(\Sigma^*) \sqrt{\log n} \cdot n^{2/q}}{\sqrt{m}}\Big).
\]
\end{lemma}
A key technical lemma that helps to establish the above bound is stated below.
\begin{lemma}
\label{lem:mean_with_X_bound}
Let $A_1, A_2, \dots, A_m \in \mathbb{R}^n$ be generated i.i.d. from $\mathcal{N}(\mu, \Sigma^*)$. Let $A$ be the $n \times m$ matrix with the columns being the points $A_i$. Let $X$ be a solution to the SDP in program~\eqref{sdp:spike} and let $B$ be any matrix, potentially chosen based on $A$, with $\|B_j\|_q \leq \delta ~ \forall j \in [m]$. Then with probability at least $1-\frac{1}{\poly(n)}$ we have that
\begin{align}
\frac 1 m \Big|\iprod{(A-\E[A])B^T, X} \Big| &\leq O(\sqrt{r \|\Sigma^*\|} \kappa \delta) + O(\kappa^2 \delta^2) + O\Big(\frac{r\kappa^2\|\Sigma^*\| \sqrt{\log n} \cdot n^{2/q}}{\sqrt{m}}\Big).
\end{align}
\end{lemma}

We defer the proof of Lemma~\ref{lem:sdp_recovery_gen} to Section~\ref{sec:proof_sdp_rec_gen}. The second step of the proof 
lower bounds the correlation of the SDP solution to $\Pi^*$ in terms of the value obtained by the SDP solution on the ideal instance $\Sigma^*=\E[AA^\top]$. 
This is established in following claim whose proof is deferred to Section~\ref{sec:distance_X_proj}.
\begin{claim}\label{clm:distance_X_proj}
Given a PSD matrix $\Sigma^*$, let $\Pi^*$ be the projection matrix on to the top $r$ eigenspace of $\Sigma^*$. For any matrix $X$ with $tr(X)=r$ and $0 \preceq X \preceq I$, it holds that 
\[
\langle X, \Pi^* \rangle \ge r - \frac{\langle \Pi^*, \E[AA^\top] \rangle - \langle X,\E[AA^\top] \rangle}{\lambda_{r}-\lambda_{r+1}} = r - \frac{\langle \Pi^*,\Sigma^* \rangle - \langle X,\Sigma^* \rangle}{\lambda_{r}-\lambda_{r+1}}.
\]
where $\lambda_{r}$ and $\lambda_{r+1}$ denote the $r$th and the $(r+1)$th largest eigenvalues of $\Sigma^*$ respectively.
\end{claim}
The above claim helps us argue that by truncating $X$ to its top-$r$ subspace we get a good approximation to $\Pi^*$. Finally, in the theorem below we show how to recover the top-$r$ principal component $\Sigma^*$ given $\wt{\Pi}$ that is a good estimate of $\Pi^*$.
\begin{theorem}\label{thm:recover_top_covariance}
Let $A_1,\ldots,A_m$ be data points drawn independently from $\calN(0,\Sigma^*)$ where the covariance matrix $\Sigma^*=\sum_{i=1}^n \lambda_i v_i v_i^{\top}$ with $\lambda_1 \ge \lambda_2 \ge \cdots \ge \lambda_n$. Let $\Sigma_{\bg} = \sum_{i=1}^r \lambda_i v_i v_i^{\top}$ and $\Pi^*$ denote the projection matrix on to the eigenspace of $\Sigma_{\bg}$. Furthermore, let $\Pi$ be a rank $r$ projection matrix with $\|\Pi-\Pi^*\|_F^2 \le \eps$. Then given a delta perturbation $\wt{A}_1,\ldots,\wt{A}_m$, with probability at least $0.99$ (over $A_1,\ldots,A_m$), the matrix $\wt{\Sigma}_{\bg}=\Pi \frac{1}{m}(\sum_{i=1}^m \wt{A}_i \wt{A}_i^{\top}) \Pi$ satisfies 
\[
\|\wt{\Sigma}_{\bg}-\Sigma_{\bg}\|_F^2 = O(\lambda_1^2 \eps +  \frac{\lambda_1^2 r^2}{m} + \kappa^4 \delta^4 + \lambda_1 \cdot \kappa^2 \delta^2) \text{ when } m = \Omega( \lambda_1^2 r^2).
\]
\end{theorem}
We end the section with the proof of our main theorem (Theorem~\ref{thm:comp_upper_gen}) using the supporting claims discussed. We defer all other proofs to Appendix~\ref{sec:comp-upper-app} and Appendix~\ref{sec:comp_upper_app}.

\begin{proofof}{Theorem~\ref{thm:comp_upper_gen}}
Recall that we define $E = \frac{1}{m} \wt{A}\wt{A}^\top - \E[AA^\top]$. Let $X$ be the solution to the SDP in \eqref{sdp:spike}. From the optimality of $X$ we have that 
\[\langle X, \Sigma^* + E \rangle \ge \langle \Pi^*, \Sigma^* + E \rangle.    
\]
We bound $\langle X, E \rangle$ and $\langle \Pi^*,E \rangle$ by $\Delta:= O\Big( \sqrt{r} \kappa \delta \sqrt{\lambda_{1}} + \kappa^2 \delta^2+ \frac{r \kappa^2 \cdot \lambda_{1} \sqrt{\log n} \cdot n^{2/q}}{\sqrt{m}}\Big)$ using Lemma~\ref{lem:sdp_recovery_gen}. Hence we get that
$
\langle X, \Sigma^* \rangle \ge \langle \Pi^*, \Sigma^*  \rangle - 2\Delta.    
$
Then we apply Claim~\ref{clm:distance_X_proj} to obtain
\begin{equation}\label{eq:sdp_guarantee}
\langle X, \Pi^* \rangle \ge r - 2\Delta/(\lambda_r - \lambda_{r+1}) = r - 2\Delta/\theta,
\end{equation}
where $\theta :=\lambda_r - \lambda_{r+1}$. 
Let $X = \sum_{i=1}^n \lambda_i(X) u_i u_i^{\top}$ be the eigendecomposition of $X$ with $\lambda_1(X) \ge \lambda_2(X) \ge \cdots \ge \lambda_n(X)$ and let $\wt{\Pi}=\sum_{i=1}^r u_i u_i^{\top}$. Since $\Pi^*$ is a projection matrix, equation~\eqref{eq:sdp_guarantee} implies that
\[
\langle \Pi^*, X \rangle = \sum_{i=1}^n \lambda_i(X) \cdot \|\Pi^* u_i \|^2_2 \ge r-2 \Delta/\theta ~\text{ and }~ \langle \Pi^*, \wt{\Pi} \rangle = \sum_{i=1}^r \|\Pi^* u_i \|^2_2.
\]
Similarly since $\langle \wt{\Pi}, X \rangle \ge \langle \Pi^*, X \rangle \ge r - 2 \Delta/\theta$, we have that
\[
\sum_{i=1}^r \lambda_i(X)=\langle \wt{\Pi}, X \rangle \ge r - 2 \Delta/\theta.
\]
\noindent At the same time from the constraints of the SDP $\sum_{i=1}^n \lambda_i(X)=\tr(X)=r$. Hence 
$$
\sum_{i=r+1}^n \lambda_i(X) \cdot \|\Pi^* u_i \|^2_2 \le \sum_{i=r+1}^n \lambda_i(X) \le 2 \Delta/\theta.
$$

Using the above we get 
\begin{align*}
\langle \Pi^*, \wt{\Pi} \rangle &=\sum_{i=1}^r \|\Pi^* u_i \|^2_2 \ge \sum_{i=1}^r \lambda_i(X)  \|\Pi^* u_i \|^2_2 \\  &=\sum_{i} \lambda_i(X)  \|\Pi^* u_i \|^2_2 - \sum_{i=r+1}^n \lambda_i(X)  \|\Pi^* u_i \|^2_2 \ge r - \frac{4 \Delta}{\theta}. 
\end{align*}
This establishes $\|\wt{\Pi}^{\bot} \Pi^*\|_F^2= \frac{1}{2}\norm{\wt{\Pi} - \Pi^*}_F^2$ is at most $4\Delta/\theta$.

Finally we note $\lambda_r(X) \ge 1-2 \Delta/\theta$ since
$\sum_{i=1}^r (1-\lambda_i(X)) \le 2 \Delta/\theta$, which implies $\|\wt{\Pi}\|_{q \to 2} \le \|X\|_{q \to 2}/(1-2 \Delta/\theta)=O(\kappa)$. The correctness of $\wt{\Sigma}_{\bg}$ then follows from Theorem~\ref{thm:recover_top_covariance}. Note that $\lambda_1^2 r^2/m$ and $\kappa^4 \delta^4$ are always less than $\lambda_1^2 \eps$ and $\lambda_1 \cdot \kappa^2 \delta^2$ separately given our parameters.
\end{proofof}

\section{Statistical Lower Bound  and  Instance-Optimality}\label{sec:lower_bound}

We now describe the construction that establishes Theorem~\ref{thm:inf_recover_lower}, the instance-optimal lower bound for recovering the principal subspace of a covariance matrix under adversarial perturbations. 
Recall that we have an arbitrary covariance matrix $\Sigma^*$ with eigendecomposition $\Sigma^*=\sum_{i=1}^n \lambda_i v_i v_i^\top$ and $\Pi^*=\sum_{i=1}^r v_i v_i^\top$ being the projection matrix onto its top-$r$ subspace. We construct based on $\Pi$ another rank-$r$ projection matrix $\Pi'$ (and  a corresponding $\Sigma'$) s.t. 
$$\norm{\Pi' - \Pi^*}_F^2 \ge \frac{c\sqrt{r} \kappa \delta}{\sqrt{\lambda_1} \log(rm) \log n} \text{ and } \norm{\Sigma_{\bg}-\Sigma'_{\bg}}_F^2 = \Omega\Big(\frac{\lambda_1^2+\dots+\lambda_r^2}{r} \cdot \norm{\Pi' - \Pi^*}_F^2 \Big),$$
and $\norm{\Pi'}_{\infty \to 2} \le (1+o(1))\kappa$. 
Moreover, for any data matrix $A$ composed of $m$ samples generated from $\calN(0,\Sigma)$, we prove that with high probability, $\exists$ a coupled data matrix $A' \in \R^{n \times m}$ generated from  $\calN(0,\Sigma')$ satisfying $\norm{A_j - A'_j}_\infty \le \delta$.

We remark that our construction also extends in a straightforward fashion to general $\ell_q$ norms to also give the same asymptotic lower bound of $\tilde{\Omega}(\sqrt{r/\lambda_1}\cdot  \kappa \delta)$, where the $\tilde{\Omega}$ hides polylogarithmic factors. We sketch the differences in the intermediate claims between the $\ell_\infty$ and general $\ell_q$ norm in the appendix (see Section~\ref{sec:lower:details}).
\anote{Commented out theorem statement since it's repeated.}
To interpret the results, let $\lambda_1 = O(1)$, and let $\kappa \gg r$ (say $\kappa=n^{0.2}$ and $r=n^{0.1}$).  
The theorem gives a lower bound of $\tilde{\Omega}(\sqrt{r} \kappa \delta)$, which is meaningful when $\kappa \delta \le \sqrt{r}$; also $\delta$ can not be too small. The range of $\delta$ is quite natural (for e.g., it is $[n^{-0.85}, n^{-0.15}]$ for the above setting). 
Theorem~\ref{thm:inf_recover_lower} shows that the upper bounds are optimal up to poly-logarithmic factors for {\em every} principal subspace $\Pi^*$ with $\norm{\Pi^*}_{\infty \to 2}=\kappa$.   
The lower bound does not have the optimal dependence in terms of the gap between the eigenvalues $(\lambda_r - \lambda_{r+1})/\lambda_1$. Please also see Theorem~\ref{thm:minmax_lower} in the appendix for a simpler minimax lower bounds that achieves the correct dependence on the eigengap as well. 

\paragraph{Construction.} To construct $\Pi'$ we take the basis vectors $v_1, \dots, v_r$ and add carefully chosen small perturbations $u_1, \dots, u_r$ to them to get a new basis $v'_1, \dots, v'_r$. Set $k':= \sqrt{\tfrac{\lambda_1}{r}}\cdot \big(\frac{\kappa}{\delta}\big)$   and   $\eps := \tfrac{c}{\log(rm)\log n}(\delta \kappa/ \sqrt{r \lambda_1}) $
for a small constant $c>0$. Note that $\eps \in [0,\tfrac{1}{4})$ 
and $2r \le k' \le n/r$ 
from our choice of parameters.
 Let $S_1, S_2, \dots, S_r \subset \set{1,\dots,n}$ be arbitrary disjoint subsets of size $k'$ each. 
Let for each $\ell \in [r]$, $T_\ell$ denote the subspace of dimension $d_\ell \ge k'-r \ge k'/2$ that corresponds to the subspace of $\R^{S_\ell}$ that is orthogonal to $\Pi^*$ and let $\Pi^{\perp}_\ell \in \R^{n \times n}$ be its projector. Then we define the eigenvectors $v'_1,\ldots,v'_r$ of $\Sigma'$, while $v'_{r+1}=v_{r+1},\ldots,v'_n=v_n$. 
\begin{align}
\forall \ell \in [r], &~~ u_\ell = \Big(\frac{1}{\sqrt{d_\ell}}\Big) \Pi^{\perp}_\ell g_\ell, ~ \text{ where } g_\ell \sim N(0,I_{n \times n}) ~~\text{independently}.\\
\text{ Define, }\forall \ell \in [r],&~~v'_\ell= (1-\eps) v_\ell + \Big(\frac{\sqrt{2\eps-\eps^2}}{\norm{u_\ell}_2}\Big) ~ u_\ell. \label{eq:lb:udef}
\end{align}
\noindent Let $\Pi'$ be the orthogonal projector on the subspace spanned by $v'_1, \dots, v'_\ell$. 
Recall
$\forall j \in [m], A_j = \sum_{\ell = 1}^n \zeta^{(j)}_\ell \sqrt{\lambda_\ell} \cdot v_\ell$ 
where $\zeta^{(j)}_\ell \sim N(0,1)$.
We construct the alternate dataset $A'$:
\begin{equation}\label{eq:constr:newA}
    A'_j = \sum_{\ell = 1}^r \zeta^{(j)}_\ell \sqrt{\lambda_\ell} \cdot \left(v_\ell + \Big(\frac{\sqrt{2\eps-\eps^2}}{(1-\eps) \norm{u_\ell}_2}\Big) u_\ell \right)+ \sum_{\ell = r+1}^n \zeta^{(j)}_\ell \sqrt{\lambda_\ell} \cdot v_\ell.
\end{equation}
(Note that the randomness in $A_j$ and $A'_j$ are coupled using the random variables $\set{\zeta^{(j)}_\ell: \ell \in [r]}, j \in [m]$.)
Observe that each sample $A'_j$ is also drawn independently from $\calN(0,\Sigma')$ with
$$\Sigma'= \sum_{\ell=1}^r \lambda_\ell \Big(v_\ell + \big(\frac{\sqrt{2\eps-\eps^2}}{(1-\eps)\norm{u_\ell}_2}\big) u_\ell \Big)\Big(v_\ell + \big(\frac{\sqrt{2\eps-\eps^2}}{(1-\eps)\norm{u_\ell}_2}\big) u_\ell \Big)^\top + \sum_{\ell=r+1}^n \lambda_{\ell} v_{\ell} v_{\ell}^{\top}.  $$ 
Its best rank-$r$ approximation is $\Sigma'_{\bg}:=\frac{1}{(1-\eps)^2}\sum_{\ell=1}^r \lambda_\ell v'_\ell (v'_\ell)^\top$, where $v'_{\ell}$ is defined in~\eqref{eq:lb:udef}.
Moreover $v'_1, \dots, v'_r$ are orthonormal (since $u_1, \dots, u_r$ are mutually orthonormal and orthogonal to $\Pi^*$). Hence $\Pi'=\sum_{\ell=1}^r v'_\ell (v'_\ell)^\top$, and the top $r$ eigenvalues of $\Sigma'$ are $\set{\lambda_\ell/(1-\eps)^2:\ell \in [r]}$.   

For our construction to work the $u_i$ vectors must simultaneously satisfy a few properties. They must be (i) orthogonal to the given $\Pi^*$, (ii) have disjoint support, (iii) be sufficiently sparse, and (iv) and have sufficiently small $\ell_\infty$ norm. Ensuring these properties requires a careful balancing act, and the following lemma gives an appropriate random distribution that satisfies these properties. 
\begin{lemma}\label{lem:lb:uproperties}
The vectors $u_1, u_2, \dots, u_r \in \R^n$ have disjoint supports $S_1, S_2, \dots, S_r \subset [n]$, and  $\Pi^* u_1 = \Pi^* u_2 =\dots = \Pi^* u_r =0$. Moreover given $k'\ge 2r$, for any $\eta<1$, with probability at least $(1 - \eta)$ we have
    \begin{align}
        \forall \ell \in [r], ~~~~~& \Big| \norm{u_\ell}_2^2 -1 \Big| \le 3\sqrt{ \log(r/\eta)/k'}+ 4 \log(r/\eta)/k' \label{eq:lb:uproperties:1}\\
         & \norm{u_\ell}_\infty \le 3\sqrt{\log(r k'/\eta)/k'}. \label{eq:lb:uproperties:2}
         ~~\text{ and }~~ \norm{u_\ell}_1 \le 2 \sqrt{k'}.
    \end{align}
\end{lemma}
The final hurdle in the construction comes from arguing that $\norm{\Pi'}_{\infty \to 2}$ is comparable to $\norm{\Pi}_{\infty \to 2}$. We argue this by analyzing the related $\norm{\Pi'}_{\infty \to 1}$ norm instead which is known to have good monotonicity properties (see Lemma~\ref{lem:monotone}), and by using properties of $v_1, \dots, v_r$ that follow from $\norm{\Pi^*}_{\infty \to 2}=\kappa$. 
Please see Section~\ref{app:lbproof} for the proof of the theorem, and Section~\ref{app:lb:aux} for proofs of the related lemmas.

\section*{Acknowledgement}
The authors would like to think Sivaraman Balakrishnan for several helpful discussions, and for suggesting the thresholding algorithm for mean estimation. 

\bibliographystyle{apalike}

\bibliography{main}
\appendix
\section{Related Work}
\label{sec:related}

\paragraph{Robustness in Supervised Learning.} In the context of supervised learning problems such as classification and regression various models of robustness have been studied in the literature. These include the classical random classification noise model \citep{angluin1988learning}, the statistical query model \citep{kearns1998efficient}, and the agnostic learning \citep{kearns1994toward} framework for modeling corruptions to the training labels. Model such as malicious noise \citep{kearns1993learning} and nasty noise \citep{diakonikolas2018learning} study settings where both the training data and the training labels could be corrupted. Typically these models assume that only a small $\epsilon$ fraction of the training data can be corrupted by an adversary. The study of these models has been very fruitful leading to a variety of algorithmic insights \citep{blum1998polynomial, dunagan2008simple, kalai2008agnostically, klivans2009learning, kalai2008agnostic, kalai2012reliable, awasthi2014power, diakonikolas2018learning}.

Recently, motivated from properties of deep neural networks, there has also been a lot in interest in modeling robustness to adversarial perturbations of the test input \citep{madry2017towards, schmidt2018adversarially, nakkiran2019adversarial, khim2018adversarial, yin2018rademacher, tsipras2018robustness, awasthi2019robustness}. While these works also model the noise as $\ell_p$ perturbations to the input, the theory of test time robustness is poorly understood and we lack provably robust algorithms for many fundamental tasks.

\paragraph{Robustness in Unsupervised Learning.} There is a large body of literature in the machine learning and statistics community on the design and study of robust algorithms for unsupervised learning tasks. Perhaps the most popular and widely studied model in this context is Huber's $\epsilon$-contamination model \citep{huber2011robust}. Here is it assumed that a given data set is generated from a mixture: $(1-\epsilon)P + \epsilon Q$ where $P$ is the true distribution about which we want to reason and $Q$ is an arbitrary distribution. Various works have studied the computational and statistical tradeoffs under Huber's model for fundamental tasks such as mean/covariance estimation \citep{yatracos1985rates, chen2016general, diakonikolas2019robust,  diakonikolas2018robustly, charikar2017learning, steinhardt2017resilience, balakrishnan2017computationally, li2017robust}, regression \citep{prasad2018robust, klivans2018efficient} and more general stochastic convex optimization \citep{prasad2018robust, diakonikolas2018sever}. Dutta et al.~\cite{Duttaetal} consider a notion of additive perturbation stability for Euclidean $k$-means clustering, where the optimal clustering is stable even when each point is perturbed by a small amount in $\ell_2$ norm. Our results together indicate that the $\infty \to 2$ norm of the principal may analogously capture a notion of stability for the subspace estimation problem when the perturbations are measured in $\ell_\infty$ norm (or $\ell_q$ for $q>2$). \anote{Added this line.}  

\paragraph{Principal Subspace Estimation in High Dimensions.} The results of our paper characterize the robustness to adversarial perturbations for estimating the top $r$-principal subspace of the covariance matrix in terms of the sparsity of the subspace. In the area of high dimensional statistics questions of estimating mean and covariance with rates depending on various notions of sparsity have been widely studied. These works however assume that the dataset is indeed generated from the idealized model. There is a long line work on the classical problem of sparse mean estimation in high dimensions \citep{donoho1992maximum, donoho1994minimax}. For the case of covariance estimation the sparse PCA formulation has been well studied and essentially corresponds to estimating the top principal component assuming that it is $\ell_0$ or $\ell_1$ sparse \citep{johnstone2001distribution, BerthetRigollet, aminiwainwright}. The works of \citet{VuLei12,VuLei13, ma2013sparse, Liuetal} extend this to estimating the top-$r$ principal subspace with rates depending on certain notions of sparsity of the subspace. Similar to our work, semidefinite programming~(SDP) based approaches have been proposed for such sparse estimation problems \citep{d2005direct}. 

Another related setting is the robust PCA formulation that has received significant interest in recent years \citep{de2003framework, candes2011robust, chandrasekaran2011rank}. Here one assumes that a given data matrix is the sum of a low rank matrix and a sparse matrix, i.e., the one with very few non-zero entries. In this case it can be shown that if true signal~(the low rank component) is well spread out then estimation is possible. In contrast, in our setting every data point could be corrupted and hence the data matrix $\tilde{A}$ cannot be written as the sum of low rank plus a sparse component. In fact, our characterization implies that under our model of perturbations, estimation is possible if and only if the signal is localized, i.e., is sparse.

\paragraph{Robustness in Combinatorial Settings.} There is also a large body of work in the theoretical computer science community studying robust algorithm design for various combinatorial problems such as graph partitioning, independent set etc. A popular framework that is used in such contexts is {\em semi-random} models \citep{BS92}. Semi-random models assume that the input is generated from an ideal distribution and then perturbed by an adversary in a non-worst case manner. The study of such models has led to the design of robust algorithms for many problems such as coloring \citep{BS92}, independent set \citep{FK99}, graph partitioning \citep{MMV12} and lately for machine learning problems as well \citep{MPW15, vijayaraghavan2018clustering, cheng2018semirandom, awasthi2018towards}.

\section{Preliminaries} \label{sec:prelims_app}


We discuss a few properties about the operator $p \to q$ norm, robust projections, and $\sin \Theta$ distance between subspaces and projections in this section.

A useful fact of the operator norms is the efficient approximation algorithms.
\begin{lemma}[\cite{nesterov1998semidefinite,steinberg2005computation}]\label{lem:approxalgos}
For any $q \le 2 \le p$, there exists an efficient randomized algorithm with an input matrix $A$ that approximates $\|A\|_{p \to q}$ within a constant factor $C_{p,q} \le 3$. Moreover for any $q \ge 2$, and for PSD matrices $M$, there exists polynomial time algorithms that approximates $\norm{M}_{q \to q^*}$ within a $1/\gamma_{q^*}^2$ factor where $\gamma_{q^*}$ is the expected $\ell_{q^*}$ norm of a standard normal r.v. In particular for $q=\infty$, this gives a $\pi/2$ approximation. 
\end{lemma}

One crucial property in the rounding algorithm of the convex program~\eqref{sdp:spike} is the monotonicity of $q \to q^*$ norm stated below (See Section 5 in \cite{ACCV} for a proof, and counter examples for other norms).
\begin{lemma}\label{lem:monotone}
For any $q>2$, $q \to q^*$ norm is monotone for PSD matrices: for any $A, B \succeq 0$, $\|A+B\|_{q \to q^*} \ge \|A\|_{q \to q^*}$. 
\end{lemma}

\paragraph{Robust projections.} We show basic properties of a projection matrix $\Pi$ in terms of its $q \to 2$ norm.
\begin{fact}\label{fact:operator_norm_pro}
Given any projection matrix $\Pi$ with $\|\Pi\|_{q \to 2} \le \kappa$ for $q>2$, we have the following properties.
\begin{enumerate}
    \item For any $\delta$ and vectors $u$ and $v$ with $\|u-v\|_{q} \le \delta$, $\|\Pi u - \Pi v\|_2 \le \kappa \delta$.
    \item Any vector $v$ in this subspace has $\|v\|_{q^*}/\|v\|_2 \le \kappa$. Moreover $\|\Pi\|_{q^*} \le \text{rank}(\Pi) \cdot \kappa^2$.
\end{enumerate}
\end{fact}
\begin{proof}
The first property follows from the definition of $q \to 2$ norm.

For the second property, $\|v\|_q = \|\Pi v\|_q \le \kappa \|v\|_2$ by definition. Morever, we could choose a orthonormal basis $v_1,\ldots,v_r$ for $\Pi$ such that
$\|\Pi\|_{q^*}=\|\sum_{i=1}^r v_i v_i^{\top}\|_{q^*} \le \sum_{i=1}^r \|v_i v_i^{\top}\|_{q^*}=r \kappa^2$.
\end{proof}
The constraint~\eqref{sdp:totalnorm} in the convex program essentially comes from the 2nd property in the above fact.\xnote{31st: move this sent to next sec?}

\paragraph{$\sin \Theta$ distance of subspaces.} Given two subspaces $S$ and $S^*$ of the same dimension, we always measure their distance in terms of the Frobenius norm of the $\sin {\Theta}(S,S^*)$ matrix, where $\Theta$ corresponds to the principal angles between the subspaces. This has a simple expression in terms of the projection matrices $\Pi, \Pi^*$ when both have the same rank:
\begin{equation*}
    \sin \Theta(S,S^*) = \Pi^{\perp} \Pi^*. \text{ Hence } \norm{\sin \Theta(S,S^*)}_F^2 = \norm{\Pi^{\perp} \Pi^*}_F^2 = \norm{\Pi^*}_F^2 - \iprod{\Pi, \Pi^*}= \tfrac{1}{2}\norm{\Pi - \Pi^*}_F^2.
\end{equation*} 
In particular, when we measure the distance between two projection matrices $\Pi$ and $\Pi^*$ of rank $r$, we will also use the following form 
\begin{equation}\label{eq:sin_theta}
    \norm{\sin \Theta(\Pi,\Pi^*)}_F^2=\norm{\Pi^{\perp} \Pi^*}_F^2=r-\langle \Pi,\Pi^*\rangle.
\end{equation}

\section{Solving the convex program \eqref{sdp:spike}} \label{app:sdpsolving}  
\begin{lemma}\label{lem:sdpsolving:spike}
For any $q \ge 2$, there exists a constant $c=c(q)\ge 1$ such that the following holds. There is a randomized polynomial time algorithm that given an instance $A \in \R^{n \times m}$ with an optimal solution $X^*$ to the relaxation \eqref{sdp:spike}-\eqref{sdp:robustnorm}, with high probability finds a solution $\widehat{X}$ that is arbitrarily close in objective value compared to $X^*$ such that $\norm{\widehat{X}}_{q \to q^*} \le c \kappa^2$. 
\end{lemma}

\begin{proof}
We first observe that the feasible set of the program is convex. We now show how to use the Ellipsoid algorithm to approximately it. We will design an approximate hyperplane separation oracle for \eqref{sdp:robustnorm} and \eqref{sdp:totalnorm}. The constraint \eqref{sdp:robustnorm} can be rewritten as $\iprod{yz^\top, X} \le \kappa^2$ for all $y,z \in \R^n$ such that $\norm{y}_q,\norm{z}_q \le 1$.
As described in Lemma~\ref{lem:approxalgos}, there exists SDP-based polynomial time algorithms that give constant factor $c=c(q)$ approximations for computing the $q \to q^*$ matrix operator norm. Such an approximation algorithm immediately gives a $c(q)$-factor approximate separation oracle; when $\norm{X}_{q \to q^*} > c \kappa^2$, the solution $y',z'$ output by the algorithm gives a separating hyperplane of the form $\iprod{y'(z')^\top,X} \le \kappa^2$. Finally, the constraint \eqref{sdp:totalnorm} is also convex and can be efficiently separated using the gradient at the given point $X$.
\end{proof}

\section{Computational Upper Bounds}
\label{sec:comp-upper-app}
In this section we provide proofs of the supporting clams that were used in establishing our main theorem~(Theorem~\ref{thm:comp_upper_gen}). We start with proving claims regarding the error term over all SDP solutions.


\subsection{Bounding Error over SDP Solutions}\label{sec:proof_sdp_rec_gen}
Here we provide the proof of Lemma~\ref{lem:sdp_recovery_gen}. We first state and prove a few useful claims.
\begin{claim}\label{clm:X_to_half_perturbation}
For any $X$ in $\mathcal{P}_{c(q)}$ and let $\wt{A}$ be an $\delta$-perturbation of $A$. Then we always have that
$$\|X^{1/2} (\wt{A}-A)\|_F \le \sqrt{c(q) m} \cdot \kappa \delta$$ 
and 
$$\langle (\wt{A}-A) (\wt{A}-A)^{\top}, X \rangle \le c(q) m \cdot \kappa^2 \delta^2.
$$
\end{claim}
\begin{proof}
Define $B=(\wt{A}-A)$. The norm bound $\|X\|_{q \to q^*} \leq c(q) \kappa^2$ along with the fact that for any matrix $M$, $\|M^\top M\|_{q \to q^*} = \|M\|^2_{q \to 2}$, implies that $\|X^{\frac 1 2}\|_{q \to 2} \leq \sqrt{c(q)} \cdot \kappa$. Denoting $B_i$ to be the $i$th column of $B$, we get that $\|B_i\|_{q} \le \delta$ and that
\begin{align*}
    \|X^{\frac 1 2}B\|_F^2 &= \sum_{i=1}^m \|X^{\frac 1 2}B_i\|^2
    \leq \sum_{i=1}^m c(q) \cdot \kappa^2 \delta^2 = m \cdot c(q) \kappa^2 \delta^2.
\end{align*}
Next, note that $\langle (\wt{A}-A) (\wt{A}-A)^{\top}, X \rangle = \langle B B^{\top}, X \rangle = \|X^{1/2} B\|_F^2 \le c(q) \cdot m  \kappa^2 \delta^2$.
\end{proof}

We will also use the following standard fact extensively.
\begin{fact}\label{fact:eig_trace_ineq}
For any two PSD matrices $A$ and $B$, $\lambda_{\min}(A) \cdot \tr(B) \le \iprod{A,B} \le \lambda_{\max}(A) \cdot \tr(B)$.
\end{fact}
\begin{proof}
We rewrite $\langle A,B \rangle = \|A^{1/2} B^{1/2}\|_F^2$, which is sandwiched by
$\lambda_{\min}(A^{1/2})^2 \cdot \|B^{1/2}\|_F^2 = \lambda_{\min}(A) \cdot \tr(B)$
and $\lambda_{\max}(A^{1/2})^2 \cdot \|B^{1/2}\|_F^2 = \lambda_{\max}(A) \cdot \tr(B)$.
\end{proof}

We will use the following standard concentration bound on the moments of the covariance matrix of Gaussian random variables (see Lemma 8.12 in \cite{ACCV} for a proof).

\begin{lemma}
\label{lem:covariance_concentration}
Let $A_1, A_2, \dots, A_m \in \mathbb{R}^n$ be generated i.i.d. from $\mathcal{N}(0, \Sigma^*)$. Let $A$ be the $n \times m$ matrix with the columns being the points $A_i$. Then with probability at least $1-\frac{1}{\poly(n)}$
\begin{align}
    \Bignorm{\frac{1}{m} AA^\top - \E[AA^\top]}_\infty &\leq c\frac{\|\Sigma\| \sqrt{\log n}}{\sqrt{m}} \text{ and } 
\|\frac{1}{m} A A^{\top} - \Sigma^*\|_q \le c \frac{\|\Sigma\| \cdot n^{2/q} \sqrt{\log n}}{\sqrt{m}}.
\end{align}
\end{lemma}
We now proceed to the proof of the main lemma that upper bounds $|\iprod{E,X}|$. 

\begin{proofof}{Lemma~\ref{lem:sdp_recovery_gen}}
Using that fact that $\E[A]=0$ and $B=\wt{A}-A$ we rewrite
\begin{align*}
    E&=\frac{1}{m}(A+B)(A+B)^\top - \E[AA^\top]\\
    &= \frac{1}{m} \Big( BB^\top +B (A-\E[A])^{\top} + (A-\E[A])B^{\top}+  AA^\top \Big) - \E[AA^\top].
\end{align*}
Hence we get that
\begin{equation}\label{eq:comp_error}
    |\iprod{E,X}| \leq \underbrace{\frac{1}{m} \iprod{BB^T, X}}_{T_1} + \underbrace{\frac{2}{m} \left|\Bigiprod{(A-\E[A])B^T, X}\right|}_{T_2} +  \underbrace{\left|\Bigiprod{\frac{1}{m}AA^\top  - \E[AA^\top], X}\right|}_{T_3}.
\end{equation}
Next we separately bound each of the terms above. Using Claim~\ref{clm:X_to_half_perturbation}, 
\begin{align*}
    T_1=\frac{1}{m} \iprod{BB^T, X} = \frac{1}{m} \|X^{1/2} B\|^2_F \leq  c(q) \kappa^2 \delta^2. 
\end{align*}

Using the concentration bound from Lemma~\ref{lem:covariance_concentration} on $\|\frac{1}{m}AA^\top  - \E[AA^\top]\|_q$, $t_3$ can be bounded as
\begin{align*}
    T_3=\Bigiprod{\frac{1}{m}AA^\top  - \E[AA^\top], X} &\leq \|\frac{1}{m}AA^\top  - \E[AA^\top]\|_q \cdot \|X\|_{q^*}\\
    &=O\Big( \frac{\lambda_{\max}(\Sigma^*) \cdot n^{2/q} \sqrt{\log n} \cdot  r\kappa^2}{\sqrt{m}} \Big).
\end{align*}

The second term $T_2$ in \eqref{eq:comp_error} is the crucial term to upper bound, and contributes the dominant term of $\sqrt{\lambda_1 r} \kappa \delta$ in the guarantees of Theorem~\ref{thm:comp_upper_gen}. A naive upper bound on $T_2$ can be obtained by $|\iprod{M_1, M_2}| \le \norm{M_1}_{q^*} \norm{M_2}_{q}$ as we did for $T_3$, but this leads to sub-optimal bounds that are off by factors involving $r$. The following technical claim which is a restatement of Lemma~\ref{lem:mean_with_X_bound} from Section~\ref{sec:comp_upper} crucially uses the constraint on $\norm{X}_{q \to q^*}$. Its proof is deferred to  Appendix~\ref{sec:comp_upper_app}. 
\begin{lemma}
\label{lem:mean_with_X_bound_app}
Let $A_1, A_2, \dots, A_m \in \mathbb{R}^n$ be generated i.i.d. from $\mathcal{N}(\mu, \Sigma^*)$. Let $A$ be the $n \times m$ matrix with the columns being the points $A_i$. Let $X$ be a solution to the SDP in program \eqref{sdp:spike} and let $B$ be any matrix, potentially chosen based on $A$, with $\|B_j\|_q \leq \delta ~ \forall j \in [m]$. Then with probability at least $1-\frac{1}{\poly(n)}$ we have that
\begin{align}
\frac 1 m \Big|\iprod{(A-\E[A])B^T, X} \Big| &\leq O(\sqrt{r \|\Sigma^*\|} \kappa \delta) + O(\kappa^2 \delta^2) + O\Big(\frac{r\kappa^2\|\Sigma^*\| \sqrt{\log n} \cdot n^{2/q}}{\sqrt{m}}\Big).
\end{align}
\end{lemma}

Combining the above bounds and using the fact that $\norm{X}_{q \to q^*} \leq c(q) \kappa^2$ we get that
\begin{align*}
\Delta \le c(q) \cdot O\Big(\kappa^2 \delta^2 + \sqrt{ r\lambda_{\max}(\Sigma^*)}\kappa \delta + \frac{r\kappa^2 \lambda_{\max}(\Sigma^*) \sqrt{\log n} \cdot n^{2/q}}{\sqrt{m}}\Big).
\end{align*}
\end{proofof}



\subsection{Bounding Correlation with the Subspace [Proof of Claim~\ref{clm:distance_X_proj}]}\label{sec:distance_X_proj}
In this section we provide the proof of Claim~\ref{clm:distance_X_proj}. For convenience, let $\eps$ denote the gap $\eps:=r - \langle X, \Pi^* \rangle$. Hence the goal is to show $\eps \le (\langle \Pi^*, \Sigma^* \rangle - \langle X,\Sigma^* \rangle)/(\lambda_r - \lambda_{r-1})$. To show this we will obtain an upper bound $\iprod{X, \Sigma^*}$ in terms of $\eps, (\lambda_r - \lambda_{r+1})$ and $\iprod{\Pi^*, \Sigma^*}$.

Given the eigen-decomposition $\Sigma^*=\sum_{i=1}^n \lambda_i v_i v_i^{\top}$ with $\lambda_1 \ge \cdots \ge \lambda_n$, we define $\Sigma_{\bg}=\sum_{i=1}^r \lambda_i v_i v_i^{\top}$ and $\Sigma_{\sm}=\sum_{i=r+1}^n \lambda_i v_i v_i^{\top}$. Note $\langle \Pi^*, \Sigma^* \rangle=\langle \Pi^*,\Sigma_{\bg} \rangle = \tr(\Sigma_{\bg})$. We will upper bound $\langle X, \Sigma^* \rangle$ as $\langle X, \Sigma_\bg + \Sigma_{\sm} \rangle$ given $\langle X, \Pi^* \rangle = r - \eps$. Let $V=[v_1,\ldots,v_n]$ denote the matrix with columns being the eigenvectors of $\Sigma^*$. For convenience, we rewrite 
\begin{align*}
\Sigma^* & =\sum_{j=1}^n \lambda_j v_j v_j^{\top} = V   \diag[\lambda_1,\ldots,\lambda_n]  V^{\top}, \\ 
\Sigma_{\bg} & = V   \diag[\lambda_1,\ldots,\lambda_r,0,\ldots,]  V^{\top}, \\
\Pi^* & = V  \diag[1,\ldots,1,0,\ldots,]  V^{\top}.
\end{align*}
The above implies that
\[
r-\eps = \langle X, \Pi^* \rangle = \langle X, V  \diag[1,\ldots,1,0,\ldots,]  V^{\top} \rangle = \langle X', \diag[1,\ldots,1,0,\ldots,] \rangle 
\]
where $X'=V X V^{\top}$. Similarly, $\langle X, \Sigma^* \rangle = \langle X', \diag[\lambda_1,\ldots,\lambda_r,0,\ldots] \rangle$. Since $X'$ also satisfies $0 \preceq X' \preceq I$, we have that
\[
\langle X, \Sigma_{\bg} \rangle = \langle X', \diag[\lambda_1,\ldots,\lambda_r,0,\ldots] \rangle \le \tr(\Sigma_\bg) - \eps \cdot \lambda_{r}
\] 
as $\langle X', \diag[1,\dots,1,0,\dots,0] \rangle = r - \eps$. Similarly, we have $\langle X', \diag[0,\dots,0,1,\dots,1] \rangle = \eps$, so 
 \[
\langle X, \Sigma_{\sm} \rangle = \langle X', \diag[0,\ldots,0,\lambda_{r+1},\ldots,\lambda_n] \rangle \le \eps \cdot \lambda_{r+1}.
\]
 The above two bounds show that
\[
\langle X,\Sigma^* \rangle \le \tr(\Sigma_\bg)-\eps \lambda_r + \eps \lambda_{r+1}.
\]
Hence we get that
\begin{align*}
\langle X, \Sigma^* \rangle &\le \langle \Pi^*, \Sigma^* \rangle - \eps(\lambda_r- \lambda_{r+1}),\\
\eps &\le \frac{\langle \Sigma, \Pi^* \rangle - \langle \Sigma^*, X \rangle}{\lambda_{r}-\lambda_{r+1}}.
\end{align*}


\subsection{Covariance Matrix Recovery}\label{sec:recover_covariance}
We end the section by showing how to recover a good approximation to the top-$r$ subspace of $\Sigma^*$ given a good approximation to $\Pi^*$. As stated before this is formalized in Theorem~\ref{thm:recover_top_covariance} which we prove below. We first state the following standard fact to bound the Frobenius error of the covariance estimation (see Theorem~4.7.1 in \cite{Vershynin} for a proof).
\begin{fact}\label{fact:covariance_est}
Let $\Sigma^*$ be a covariance matrix of rank $r$ and largest eigenvalue $\lambda_1$. For any $m$, and vectors $A_1,\ldots,A_m \sim N(0,\Sigma^*)$, it holds with probability at least $1-10^{-3}$, that $\|\wt{\Sigma}-\Sigma^*\| = \lambda_1 \cdot  O(\sqrt{r/m}+r/m)$ for $\wt{\Sigma}=\frac{1}{m} \sum_i A_i A_i^{\top}$. Moreover, if $m=O(\lambda_1^2 r^2/\beta)$ then with prob. at least $1-10^{-3}$, $\|\wt{\Sigma}-\Sigma\|_{F}^2 \le \beta$ for $\beta<r$.
\end{fact}

We also use the following lemma showing how to recover a good approximation to the top-$r$ subspace of $\Sigma^*$ in the absence of noise.

\begin{lemma}\label{lem:recover_sigma}
For any covariance matrix $\Sigma^*=\sum_{i=1}^n \lambda_i v_i v_i^{\top}$ with eigenvalues $\lambda_1 \ge \lambda_2 \ge \cdots \ge \lambda_n$, let $\Sigma_{\bg}=\sum_{i=1}^r \lambda_i v_i v_i^{\top}$ and $\Pi^*$ be the projection matrix on to the top $r$ eigenspace of $\Sigma^*$. Given any rank $r$ projection matrix $\Pi$ with $\|\Pi^*-\Pi\|_F^2 \le \eps$, and $m = \Omega( \lambda_1^2 \cdot r^2 )$, we have that with probability at least $1-10^{-3}$, $\|\wt{\Sigma} - \Sigma_{\bg} \|_F^2 =O( \lambda_1^2 \cdot \eps+\frac{\lambda_1^2 r^2}{m})$ for $\wt{\Sigma}=\Pi (\frac{1}{m} \sum_{i=1}^m A_i A_i^{\top}) \Pi$ and $A_1, \dots, A_m$ are generated i.i.d. from $N(0,\Sigma^*)$.
\end{lemma}
The above lemma is an extension of Lemma 8.2 in \citep{ACCV}. For completeness, we provide the proof in Appendix~\ref{sec:comp_upper_app}. 
Next we establish Theorem~\ref{thm:recover_top_covariance} showing covariance recovery in the presence of adversarial perturbations.

\begin{proofof}{Theorem~\ref{thm:recover_top_covariance}}
For the estimate 
$\Pi (\frac{1}{m} \sum_{i=1}^m \wt{A}_i \wt{A}_i^{\top}) \Pi$ output by the algorithm we have that
\begin{align*}
    & \|\Pi (\frac{1}{m} \sum_{i=1}^m \wt{A}_i \wt{A}_i^{\top}) \Pi - \Pi^* \Sigma^* \Pi^*\|^2_F\\
    \le & 2 \|\Pi (\frac{1}{m} \sum_{i=1}^m \wt{A}_i \wt{A}_i^{\top}) \Pi - \Pi (\frac{1}{m} \sum_{i=1}^m A_i A_i^{\top}) \Pi\|^2_F + 2\|\Pi (\frac{1}{m} \sum_{i=1}^m A_i A_i^{\top}) \Pi - \Pi^* \Sigma^* \Pi^*\|^2_F\\
    \le & 2\|\Pi (\frac{1}{m} \sum_{i=1}^m \wt{A}_i \wt{A}_i^{\top}) \Pi - \Pi (\frac{1}{m} \sum_{i=1}^m A_i A_i^{\top}) \Pi\|^2_F + O(\lambda_1^2 \eps + \frac{\lambda_1^2 r^2}{m}),
\end{align*}
where we use Lemma~\ref{lem:recover_sigma} in the last step. Let $\wt{A}_i=A_i+B_i$ such that $B_i$ is the perturbation of the $i$th data point. We can rewrite the first term above as
\begin{align*}
& \|\Pi (\frac{1}{m} \sum_{i=1}^m \wt{A}_i \wt{A}_i^{\top}) \Pi - \Pi (\frac{1}{m} \sum_{i=1}^m A_i A_i^{\top}) \Pi\|_F \\
= & \|\Pi (\frac{1}{m} \sum_{i=1}^m (A_i + B_i) (A_i + B_i)^{\top} \Pi - \Pi (\frac{1}{m} \sum_{i=1}^m A_i A_i^{\top}) \Pi\|_F\\
\le & \|\Pi (\frac{1}{m} \sum_{i=1}^m B_i B_i^{\top}) \Pi\|_F + \| \Pi (\frac{1}{m} \sum_{i=1}^m A_i B_i^{\top}) \Pi\|_F + \| \Pi (\frac{1}{m} \sum_{i=1}^m B_i A_i^{\top}) \Pi\|_F.
\end{align*}
Now we bound each Frobenius norm separately. For the first term we have
\[\|\Pi (\frac{1}{m} \sum_{i=1}^m B_i B_i^{\top}) \Pi\|_F \le \frac{1}{m} \sum_{i=1}^m \| \Pi B_i B_i^{\top} \Pi \|_F = O(\kappa^2 \delta^2)
\]
where we have used the fact that $\Pi B_i$ is a vector of norm at most $O(\kappa \delta)$. We bound the second term $\| \Pi (\frac{1}{m} \sum_{i=1}^m A_i B_i^{\top}) \Pi\|_F$ (and similary for the third one), by
\[
\frac{1}{m} \|\Pi A\| \cdot \|B^{\top} \Pi\|_F \le \frac{1}{m} \cdot \sqrt{\lambda_1 m} \cdot O(1 + \sqrt{r/m})  \cdot \sqrt{m} \kappa \delta = \sqrt{\lambda_1} \cdot O(\sqrt{r/m}+1) \kappa \delta
\]
where we bound $\|B^{\top} \Pi\|_F^2 \le \sqrt{m} \kappa \delta$ from the above bound on $\|\Pi (\frac{1}{m} \sum_{i=1}^m B_i B_i^{\top}) \Pi\|_F$ and $\|\Pi A\| \le \sqrt{\lambda_1 m} \cdot (1 + \sqrt{r/m})$ as follows: rank$(\Pi A)=r$ and Fact~\ref{fact:covariance_est} implies that with probability at least  $1-10^{-3}$, $\|\Pi \frac{1}{m} A A^{\top} \Pi - \E[(\Pi A)\cdot (\Pi A)^{\top}]\| \le O(\lambda_1 \cdot \sqrt{r/m})$. Since $\| \E[(\Pi A)\cdot (\Pi A)^{\top}] \| \le \|\E[A A^{\top}]\| \le \lambda_1$, $\|\Pi \frac{1}{m} A A^{\top} \Pi\| \le \lambda_1 + \lambda_1 \cdot O(\sqrt{r/m})$.
\pnote{I don't understand the two lines above.}

Combining the above bounds we get that $\|\Pi (\frac{1}{m} \sum_{i=1}^m \wt{A}_i \wt{A}_i^{\top}) \Pi - \Pi^* \Sigma^* \Pi^*\|^2_F$ can be bounded by
\[
O(\lambda_1^2 \eps + \lambda_1^2 \cdot r^2/m) + O(\kappa^4 \delta^4) + O(\lambda_1 \cdot (1+r/m) \cdot \kappa^2 \delta^2).
\]
\end{proofof}

\section{Additional Proofs from Section~\ref{sec:comp_upper}}
\label{sec:comp_upper_app}
\begin{proofof}{Lemma~\ref{lem:mean_with_X_bound_app}}
We use the fact that for matrices $M_1, M_2, Q_1,$ and $Q_2$, it holds that
$$
\iprod{M_1 M_2, Q_1 Q_2} \leq \|M_1^\top Q_1\|_F \|M_2 Q_2^\top\|_F
$$
to rewrite
\begin{align*}
    \frac{1}{m}\iprod{(A-\E[A])B^T, X} &= \frac{1}{m}\iprod{(A-\E[A])B^T, X^{\frac 1 2}X^{\frac 1 2}}\\
    &\leq \frac{1}{m} \|(A-\E[A])^\top X^{\frac 1 2}\|_F \|X^{\frac 1 2} B \|_F
\end{align*}

By Claim~\ref{clm:X_to_half_perturbation}, $\|X^{\frac 1 2} B \|_F \le \sqrt{m} \kappa \delta$. Note that $
\|(A-\E[A])^\top X^{\frac 1 2}\|^2_F= \langle A A^{\top} ,X\rangle$ given $\E[A]=0$. Then we split it into 
\[
\langle A A^{\top} ,X\rangle = \langle A A^{\top} - m \cdot E[A A^{\top}], X\rangle + \langle m \cdot \E[A A^{\top}], X \rangle = O\left( r\kappa^2 \cdot \|\Sigma^*\| \sqrt{\log n} \cdot n^{2/q} \cdot \sqrt{m} + \|\Sigma^*\| \cdot r m \right),
\]
where the first bound comes from the above proof of Lemma~\ref{lem:sdp_recovery_gen} for the last term in \eqref{eq:comp_error} and the second bound comes from Fact~\ref{fact:eig_trace_ineq}.

We finish the proof by combining the above bounds:
\begin{align*}
    \frac{2}{m}\iprod{(A-\E[A])B^T, X} &\leq \frac{1}{m} \cdot \sqrt{m} \kappa \delta \cdot O\left( \|\Sigma^*\| \cdot rm + r\kappa^2 \cdot \|\Sigma^*\| \sqrt{\log n} \cdot n^{2/q} \sqrt{m} \right)^{1/2} \\
    & = O(\sqrt{ r \|\Sigma^*\|}\kappa \delta) + \kappa \delta \cdot O\left(\frac{ \|\Sigma^*\| r \kappa^2 \sqrt{\log n} \cdot  n^{2/q}}{\sqrt{m}}\right)^{1/2}\\
    & \le O(\sqrt{ r\|\Sigma^*\|}\kappa \delta) + O(\kappa^2 \delta^2) + O\left( \frac{\|\Sigma^*\| r \kappa^2 \sqrt{\log n} \cdot n^{2/q}}{\sqrt{m}} \right).
\end{align*}
\end{proofof}

\subsection{Proof of Lemma~\ref{lem:recover_sigma}}\label{sec:proof_recover_sigma}
We will use the following fact to apply triangle inequality.
\begin{fact}\label{fact:recover_sigma}
Given a rank $r$ covariance matrix $\Sigma^*$ with all eigenvalues upper bounded by $\lambda_{\max}$ and projection matrix $\Pi^*$, for any rank $r$ projection $\Pi$ with $\|\Pi^*-\Pi\|_F^2 \le \eps$ and any $\wt{\Sigma}$ (not necessarily rank $r$), we have 
\[
\|\Sigma^*-\Pi \wt{\Sigma} \Pi \|_F^2 \le 8 \lambda^2_{\max} \cdot \eps+2 \| \Pi \Sigma^* \Pi\ - \Pi \wt{\Sigma} \Pi \|_F^2.
\]
\end{fact}
\begin{proof}
At first, we have $\|\Sigma^* - \Pi \wt{\Sigma} \Pi \|_F^2 \le 2 \|\Sigma^* - \Pi \Sigma^* \Pi\|_F^2 + 2 \| \Pi \Sigma^* \Pi\ - \Pi \wt{\Sigma} \Pi \|_F^2$.

Since $\Pi^*$ is the projection matrix of $\Sigma^*$, we have
\[
\|\Sigma^* - \Pi \Sigma^* \Pi \|_F^2 = \|\Pi^* \Sigma^* \Pi^* - \Pi \Sigma^* \Pi \|_F^2 \le 2 ( \|\Pi^* \Sigma^* \Pi^* -\Pi \Sigma^* \Pi^* \|_F^2 + \|\Pi \Sigma^* \Pi^* -\Pi \Sigma^* \Pi \|_F^2 ).
\]
Since $\|AB\|_{F}^2 \le \|A\|^2_{op} \cdot \|B\|_F^2$, we further simplify it as 
\[
2 (\|\Pi^*-\Pi\|^2_F \cdot \|\Sigma^*\|^2_{op} + \|\Sigma^*\|^2_{op} \cdot \|\Pi^*-\Pi\|^2_F) \le 4 \lambda^2_{\max} \cdot \eps.
\]
\end{proof}

We finish the proof of Lemma~\ref{lem:recover_sigma}.

\begin{proofof}{Lemma~\ref{lem:recover_sigma}}
Given $A_i \sim N(0,\Sigma^*)$, we know $\Pi A_i$ is a random vector generated by $N\big( 0, \Pi \Sigma^* \Pi \big)$. So we apply Fact~\ref{fact:covariance_est} to bound $\|\Pi \Sigma \Pi - \Pi (\frac{1}{m} \sum_{i=1}^m A_i A_i^{\top}) \Pi \|_{F}^2 \le \delta$. Then we consider $\Sigma_{\bg}$:
\[
\|  \Pi \Sigma \Pi - \Pi \Sigma_{\bg} \Pi\ \|_F = \| \Pi \Sigma_{\sm} \Pi \|_F \le \| (\Pi-\Pi^*) \Sigma_{\sm} \Pi \|_F + \| \Pi^* \Sigma_{\sm} \Pi \|_F.
\]
Since $\Pi^*$ is the projection matrix of $\Sigma_{\bg}$, $\Pi^* \Sigma_{\sm}=0$ such that the second term becomes 0. For the first term $\| (\Pi-\Pi^*) \Sigma_{\sm} \Pi \|_F$, we upper bound it by 
\[
\| (\Pi-\Pi^*) \Sigma_{\sm} \Pi \|_F \le \|\Pi-\Pi^*\|_F \cdot \|\Sigma_{\sm}\|_{op} \cdot \|\Pi\|_{op} \le \lambda_1 \cdot \sqrt{\eps}.
\]
From the above discussion, we have 
\[
\|\Pi \Sigma_{\bg} \Pi - \Pi (\frac{1}{m} \sum_{i=1}^m A_i A_i^{\top}) \Pi \|_{F}^2 \le 2\|\Pi \Sigma \Pi - \Pi (\frac{1}{m} \sum_{i=1}^m A_i A_i^{\top}) \Pi \|_{F}^2 + 2\| (\Pi-\Pi^*) \Sigma_{\sm} \Pi \|^2_F = O( \delta + \lambda_1^2 \eps ).
\]
The final bound follows from Fact~\ref{fact:recover_sigma} with $\Sigma^*=\Sigma_{\bg}$ there.
\end{proofof}

\section{Statistical Lower Bound on the Error and Instance-Optimality} \label{app:lb}

\subsection{Auxiliary claims and Proofs.} \label{app:lb:aux}

\begin{proof}[Proof of Lemma~\ref{lem:lb:uproperties}]
By construction $u_1, \dots, u_r$ have disjoint supports, and for each $\ell \in [r]$, $\Pi^* \Pi^{\perp}_\ell=0$; hence $\Pi^* u_\ell=0$. 
We now show \eqref{eq:lb:uproperties:1}. Note that $\Pi^{\perp}_\ell g_\ell$ is distributed according to the Gaussian $\calN(0, \Pi^{\perp}_\ell)$. Hence $\E[\norm{\Pi^{\perp}_\ell g_\ell}_2^2]= \tr(\Pi^{\perp}_\ell) = d_\ell$.  For a fixed $\ell \in [r]$, using concentration bounds for $\chi^2$ distributions we have for any $t>0$
\begin{align*}
\Pr\Big[\big| \norm{u_\ell}_2^2 - 1 \big| > 2\sqrt{\frac{t}{d_\ell}} + 2 \frac{t}{d_\ell} \Big] & = \Pr\Big[\big| \norm{\Pi^{\perp}_\ell g_\ell}_2^2 - d_\ell \big| > 2\sqrt{d_\ell t } + 2 t \Big] \le \exp(-t).
\end{align*}
Substituting $t=\log(r/\eta)$, along with $d_\ell \ge k'-r \ge k'/2$ and a union bound over all $\ell \in [r]$ establishes \eqref{eq:lb:uproperties:1}. Then the last property of $\|u_{\ell}\|_1 \le 2 \sqrt{k'}$ follows from the Cauchy-Schwartz inequality with the fact that the support size of $u_{\ell}$ is at most $k'$.

Now we upper bound $\norm{u_\ell}_\infty$. For each coordinate $i$ and $\ell$, \[u_{\ell}(i)=\frac{1}{\sqrt{d_{\ell}}} \iprod{\Pi^{\perp}_\ell(i,:), g_\ell} \text{ where } \Pi^{\perp}_\ell(i,:) \text{ represents the $i$th row of } \Pi^{\perp}_\ell.\] This is a Gaussian random variable. 
Hence for a fixed $\ell \in [r]$, with probability at least  $1-\frac{\eta}{r}$,
\begin{align*}
    \norm{u_\ell}_\infty & = \frac{1}{\sqrt{d_{\ell}}} \max_{i \in [n]} |\iprod{\Pi^{\perp}_\ell(i,:), g_\ell}| \le 2 \sqrt{\log(rk'/\eta)} \cdot \frac{\max_{i \in [n]}\norm{\Pi^{\perp}_\ell(i,:)} }{\sqrt{d_{\ell}}} \le 2\sqrt{ \log(rk'/\eta)} \cdot \frac{1}{\sqrt{k'/2}},
\end{align*}
since $\Pi^{\perp}_\ell$ is an orthogonal projection matrix. After a union bound over $\ell \in [r]$, \eqref{eq:lb:uproperties:2} follows. 
\end{proof}

\begin{proof}[Proof of Lemma~\ref{lem:lb:helper}]
The proof just uses norm duality and relations between different norms.  
\begin{align*}
\Bignorm{\sum_\ell u_\ell v_\ell^\top }_{q \to q^*}&= \max_{\substack{x,y: \norm{x}_q \le 1\\ \norm{y}_q \le 1}} \sum_{\ell=1}^r \iprod{x,u_\ell} \iprod{v_\ell, y} \le \sum_\ell  \norm{u_\ell}_{q^*} |\iprod{v_\ell, y}| \\
& \le \max_\ell \norm{u_\ell}_{q^*} \cdot \sum_\ell   |\iprod{v_\ell, y}| = \max_{\ell \in [r]} \norm{u_\ell}_{q^*} \cdot \norm{V^\top}_{q \to q^*}  \\
&= \max_{\ell \in [r]} \norm{u_\ell}_{q^*} \cdot \norm{V}_{q \to q^*} \le  r^{1/2-1/q} \norm{V}_{q \to 2} \cdot \max_{\ell \in [r]} \norm{u_\ell}_{q^*}.
\end{align*}
The last inequality follows since $\norm{Vy}_{q^*} \le r^{1/2-1/q} \norm{Vy}_2$ for any $y \in \R^n$ since $V$ has $r$ columns (see Section~\ref{sec:prelims}).

For the second statement, we have $\norm{UU^\top}_{q \to q^*}= \norm{U^\top}_{q \to 2}^2=\norm{U}_{2 \to q^*}^2$ using the variational characterization of operator norms and norm duality (see Section~\ref{sec:prelims}). We now upper bound $\norm{U}_{2 \to q^*}$. Consider any $y \in \R^r$ with $\norm{y}_2=1$. Then because of the disjoint supports of the columns of $U$ 
\begin{align*}
    \norm{Uy}_{q^*}^{q^*}&= \Big(\sum_{\ell=1}^r |y_\ell|^{q^*} \norm{u_\ell}_{q^*}^{q^*} \Big) \le  \max_{\ell \in [r]} \norm{u_\ell}_{q^*}^{q^*} \cdot \norm{y}_{q^*}^{q^{*}} \\
    \norm{Uy}_{q^*}&\le \norm{y}_{q^*} \cdot \max_{\ell \in [r]} \norm{u_\ell}_{q^*} \le r^{1/q^* - 1/2} \norm{y}_2 \cdot \max_{\ell \in [r]} \norm{u_\ell}_{q^*} \le r^{1/2-1/q} \max_{\ell \in [r]} \norm{u_\ell}_{q^*}. 
\end{align*}
Hence the lemma holds. 
\end{proof}


The following simple lemma will be in upper bounding the magnitude of the perturbation for each sample point.
\begin{lemma}\label{lem:helper:perturb}
Given any $u_1, \dots, u_r \in \R^n$ with disjoint support, and any $\alpha_1, \dots, \alpha_r \in \R$,  we have
$$ \Bignorm{\sum_{\ell=1}^r \alpha_\ell u_\ell }_q \le r^{1/q} \max_{\ell \in [r]} |\alpha_\ell| \norm{u_\ell}_{q} .$$
\end{lemma}
\begin{proof}
Since $u_1, \dots, u_r$ have disjoint support,
\begin{align*}
          \Bignorm{\sum_{\ell=1}^r \alpha_\ell u_\ell }_q^q &= \sum_{\ell=1}^r |\alpha_\ell|^q \norm{u_\ell}_q^q \le r \Big(\max_{\ell \in [r]} |\alpha_\ell| \norm{u_\ell}_q \Big)^q, \text{   as required.} 
\end{align*}
\end{proof}

The following lemma is also useful to upper bound the $\infty \to 2$ operator norm of the alternate subspace projector $\Pi'$. 
\begin{lemma}\label{lem:lb:helper}
Given any vectors $u_1, \dots, u_r$ and vectors $v_1, \dots, v_r$ that form the columns of $U, V \in \R^{n \times r}$ separately, then for any $q \ge 1$
\begin{equation}\Bignorm{UV^\top}_{q \to q^*} \le \norm{V}_{q \to q^*} \Big(\max_{\ell \in [r]} \norm{u_\ell}_{q^*} \Big) \le r^{1/2-1/q} \norm{V}_{q \to 2} \Big(\max_{\ell \in [r]} \norm{u_\ell}_{q^*} \Big)  .\label{eq:helper:1}\end{equation}
Moreover if $u_1, \dots, u_r$ have disjoint support then
\begin{equation}\norm{ UU^\top}_{q \to q^*} = \norm{U}_{2 \to q^*}^2 \le r^{1-2/q} \Big(\max_{\ell \in [r]} \norm{u_\ell}_{q^*}^2 \Big)  .\label{eq:helper:2}\end{equation}
\end{lemma}

\subsection{Warmup: Min-max lower bound}\label{sec:proof_min_max}

We now give a min-max optimal lower bound. While  Theorem~\ref{thm:inf_recover_lower} is much more general, we include this argument since it is simpler and helps build intuition, and also gives the correct dependence on the eigengap. The lower bound will apply for $\Sigma^*= \theta \Pi^* + I$; hence $\Sigma_\bg=(1+\theta) \Pi^*$ and $\Sigma_\sm= (I-\Pi^*)= (\Pi^*)^{\perp}$.    
\begin{theorem}\label{thm:minmax_lower}
Suppose we are given parameters $n$, $m$, $\theta>0$, $r\in \N$, $\kappa$, and $\delta>0$ satisfying $\sqrt{r\lambda_1}(\tfrac{\kappa}{n}) <\delta \le \sqrt{r \theta}/\kappa$. 
There exist orthogonal projection matrices $\Pi^*, \Pi'$ both of rank $r$ with $\norm{\Pi^*}_{\infty \to 2} \le \kappa$ and $\norm{\Pi'}_{\infty \to 2} \le \kappa$ such that:
\begin{itemize}
\item We have the coupling data matrices $A$ and $A' \in \R^{n \times m}$ with their columns generated i.i.d. from  $\calN(0,I+\theta \Pi^*)$ and $\calN(0, I+\theta \Pi')$ respectively, such that $\norm{A - A'} \le \delta$ with high probability. 
\item $\norm{\Pi' - \Pi^*}_F^2 = \Omega\big( \frac{1}{\sqrt{\theta}} \cdot \sqrt{r} \delta \kappa/\log nm \big)$.
\end{itemize}
\end{theorem}

We now prove the above theorem. 
We first show the constructions of $\Pi'$ and $A'$. Choose $k$ to be a power of 2 in $[\kappa^2/3,2\kappa^2/3]$. Let $S:=\{ 1, 2, \cdots, k \} \subset [n]$ and $v_1, v_2, \cdots, v_r$ be any $r$ orthonormal vectors of the form $v_{\ell}(i) = \pm 1/\sqrt{k}$ if $i \in S$ and $0$ otherwise. For example, there are $k$ Fourier characters $v_{\ell}$ in $\{0,1\}^{\log k}$ that are orthogonal to each other with $\|v_{\ell}\|_{\infty} \le 1/\sqrt{k}$: For each $i \in [k]$, let $\overrightarrow{i} \in \{0,1\}^{\log k}$ be its binary form. Then each Fourier character is $v_{\ell}(i)=(-1)^{\langle \overrightarrow{\ell},\overrightarrow{i} \rangle}/\sqrt{k}$. 

Let $k' \in [\frac{1}{4},\frac{1}{2}] \cdot \sqrt{\theta} \kappa/(\delta \sqrt{r})$ be a power of 2 to denote the support size of the perturbation vector. Let $u_1, \dots, u_r$ be unit vectors supported on a disjoint set of $k'$ coordinates each from $[n] \setminus S$ with $\norm{u_\ell}_\infty = 1/\sqrt{k'}$ for each $\ell \in [r]$ using the same construction of $v_1,\ldots,v_r$. Set $\eps:=c_4 \frac{\delta \kappa}{\sqrt{r \theta} \log (nm)}$ for some small constant $c_4>0$ such that $\eps \le 1/10$ from the parameters given in the statement. Finally, let
$$\forall \ell \in [r], ~~ v'_\ell := (1-\eps) v_\ell + \sqrt{2\eps - \eps^2} u_\ell,$$
and let $\Pi'$ be the orthogonal projection onto the subspace spanned by $v'_1, \dots, v'_r$. 
Now the original data point $A_j$ and its coupling data point $A'_j$ (for $j \in [m]$) for matrices $A,A'$ are drawn i.i.d. as follows:
\begin{align}
    A_j = \sum_{\ell=1}^r \zeta_\ell v_\ell + g, ~& \text{ and }  A'_j = \sum_{\ell=1}^r \zeta_\ell v'_\ell +g,\\
    \text{ where } \forall \ell \in [r], ~ \zeta_\ell \sim \calN(0,\theta) ~& \text{ and } g \sim \calN(0,I). 
\end{align} 
Then we bound the $\infty \to 2$ norm of $\Pi^*$ and $\Pi'$.
\begin{claim}\label{clm:min_max_infty}
$\|\Pi^*\|_{\infty \to 2} \le  \kappa$ and $\|\Pi'\|_{\infty \to 2} \le  \kappa$.
\end{claim}
\begin{proofof}{Claim~\ref{clm:min_max_infty}}
We have $\Pi^* = \sum_{\ell=1}^r v_\ell v_\ell^\top $, since $v_1, \dots, v_r$ is an orthonormal basis for the subspace given by $\Pi^*$, and
\[\norm{\Pi^*}_{\infty \to 2} = \norm{\Pi^*}_{2 \to 1} = \max_{y: \norm{y}_2 =1} \norm{\Pi^* y}_1 \le \sqrt{k} \norm{\Pi^* y}_2 \le \sqrt{k} \le \sqrt{\frac{2}{3}} \kappa, \label{eq:robustnessgroundtruth} \]
where the first inequality follows from Cauchy-Schwartz inequality and the support size being bounded by $k$. 
Now we compute the $\infty \to 1$ norm of $\Pi'$.
\begin{align}
    \Pi'&= \Pi^* +  \sum_{\ell \in [r]} (-2 \eps + \eps^2) v_\ell v_\ell^\top + (2\eps-\eps^2) \sum_\ell  u_\ell u_\ell^\top + \sqrt{2 \eps-\eps^2}(1-\eps) (v_\ell u_\ell^\top + u_\ell v_\ell^\top) \nonumber \\
    \norm{\Pi'}_{\infty \to 1} & \le \norm{\Pi^*}_{\infty \to 1} + 2\eps \norm{\sum_{\ell} u_\ell u_\ell^\top}_{\infty \to 1} + 2\sqrt{2\eps} \norm{\sum_{\ell} u_\ell v_\ell^\top}, \label{eq:robustnessaltn:1}
\end{align}
due to triangle inequality and using the monotonicity of the $\infty \to 1$ norm (Lemma~\ref{lem:monotone}).

For the second term, we note $\|\sum_{\ell} u_{\ell} u_{\ell}^{\top}\|_{2 \to 1} \le \sqrt{r} \cdot \max_{\ell} \|u_{\ell}\|_1 \le \sqrt{r k'}$.

We now bound the third term using \eqref{eq:helper:1} of Lemma~\ref{lem:lb:helper}.
\begin{align*}
\Bignorm{\sum_\ell u_\ell v_\ell^\top }_{\infty \to 1}&\le\sqrt{r } \cdot \norm{V}_{\infty \to 2} \cdot \max_{\ell} \norm{u_\ell}_1 \le \sqrt{r k'}  \cdot \sqrt{\frac{2}{3}} \kappa \le \frac{1}{\sqrt{3}} (r \theta)^{1/4} \sqrt{\frac{\kappa}{\delta}} \cdot \kappa
\end{align*}
given $k'\le \frac{1}{2} \cdot \frac{\sqrt{\theta} \kappa}{\delta\sqrt{r}}$. Hence substituting in \eqref{eq:robustnessaltn:1}, and using \eqref{eq:helper:2} we have
\begin{align*}
    \norm{\Pi'}_{\infty \to 1}  &\le \frac{2}{3}\kappa^2+ 2 \eps \cdot r \max_{\ell} \norm{u_\ell}_1^2 + \max_\ell \norm{u_\ell}_1 \cdot \norm{V}_{\infty \to 1} \le \kappa^2 + 2 \eps \cdot r \kappa' + \sqrt{8\eps/3} \cdot (\theta r)^{1/4} \sqrt{\frac{\kappa}{\delta}} \cdot \kappa \\
&\le \frac{2\kappa^2}{3}+2 \cdot O\Big(\frac{\delta \kappa}{\sqrt{r \theta} \log nm}\Big) \cdot r \cdot \frac{\sqrt{\theta} \kappa}{2 \delta \sqrt{r}} + \sqrt{8/3} \cdot \sqrt{\frac{\delta \kappa}{\sqrt{r \theta} \log nm} \cdot \sqrt{r \theta} \cdot \kappa/\delta} \cdot \kappa \le \kappa^2,
 \end{align*}
given $\eps =\Theta\Big(\frac{\delta \kappa}{\sqrt{r \theta} \cdot \log nm}\Big)$.
Hence $\norm{\Pi'}_{\infty \to 2} \le \kappa$.
\end{proofof}

\begin{claim}\label{clm:min_max_frob}
$\|\Pi^*-\Pi'\|_F^2 = \Omega(\frac{\sqrt{r} \cdot \delta \kappa}{\sqrt{\theta} \cdot \log nm})$.
\end{claim}

\begin{proofof}{Claim~\ref{clm:min_max_frob}}
We lower bound the distance between the projections using the orthogonality between $u_1,\ldots,u_r$ and $v_1,\ldots,v_r$:
\begin{align*}
    \Pi' - \Pi^*&= \sum_{\ell=1}^r v'_\ell (v'_\ell)^{\top} - v_\ell v_\ell^\top \\
    &= \sum_{\ell \in [r]} (-2 \eps + \eps^2) v_\ell v_\ell^\top + (2\eps-\eps^2) \sum_\ell  u_\ell u_\ell^\top + \sqrt{2 \eps-\eps^2}(1-\eps) (v_\ell u_\ell^\top + u_\ell v_\ell^\top) \\
    \text{So, }\norm{\Pi'-\Pi^*}_F^2 & \ge r (4 \eps - 2\eps^2) = \Omega\left( \frac{\sqrt{r} \delta \kappa}{\sqrt{\theta} \log nm} \right) . 
\end{align*}
\end{proofof}

\begin{claim}\label{clm:min_max_perturbation}
With high probability, the coupling data matrix $A'$ satisfies $\|A-A'\|_{\infty}\le \delta$.
\end{claim} 
\begin{proof}
Note that $\sum_\ell \zeta_\ell v_j$ is a Gaussian with co-variance $\calN(0, \theta \Pi^*)$, and each co-ordinate of this vector is a normal R.V. with mean $0$ and variance at most $\norm{v_j}_\infty^2 \sum_\ell \zeta_\ell^2$.    
\begin{align*}
    \norm{A_j - A'_j}_\infty&\le  \eps \Bignorm{ \sum_{\ell=1}^r \zeta_\ell  v_\ell}_\infty + \sqrt{2\eps - \eps^2} \Bignorm{\sum_{\ell =1}^r  \zeta_\ell u_\ell}_\infty  \\
\text{First, } \eps \Bignorm{ \sum_{\ell=1}^r \zeta_\ell  v_\ell}_\infty    &\le 2\eps  \sqrt{\theta \cdot r \log(nm)}  \max_\ell \norm{v_\ell}_\infty  \\
    &\le 2 \cdot \Theta\Big(\frac{\delta \kappa}{\sqrt{r \theta} \cdot \log nm}\Big) \cdot \sqrt{\theta r \log (nm)} \frac{1}{\kappa} \le \frac{\delta}{2}, 
\end{align*}
when $c_4$ in $\eps$ is a small constant, and since $\norm{v_\ell}_\infty \le 1/\kappa$. Bounding the second term 
uses the fact that the $u_1, \dots, u_r$ have disjoint support, along with the upper bounds for $\norm{u_\ell}_\infty$. 
\begin{align*}
 \sqrt{2\eps - \eps^2} \Bignorm{\sum_{\ell =1}^r  \zeta_\ell u_\ell}_\infty  &\le 2\sqrt{\theta \cdot \eps \log (nm)} \max_\ell \norm{u_\ell}_\infty \\
    &\le O\left(\sqrt{\theta \log (nm) \cdot \frac{\delta \kappa}{\sqrt{r \theta} \cdot \log nm}} \right) \cdot \sqrt{\frac{\delta \sqrt{r}}{\sqrt{\theta} \kappa}} \le\frac{\delta}{2}. 
\end{align*}
Combining the two bounds, we see that $\norm{A-A'}_\infty \le \delta$ with high probability.
\end{proof}

The correctness of Theorem~\ref{thm:minmax_lower} now follows from Claim~\ref{clm:min_max_infty}, Claim~\ref{clm:min_max_frob} and Claim~\ref{clm:min_max_perturbation}.

\subsection{Asymptotic Instance-Optimal Lower Bound: Proof of Theorem~\ref{thm:inf_recover_lower}} 
\label{app:lbproof}
\begin{proofof}{Theorem~\ref{thm:inf_recover_lower}}
We now establish the required properties of $\Pi'$. Firstly $u_1, \dots, u_r$ are orthogonal to each other and to $\Pi^*$ (i.e.,  $v_1, \dots, v_r$). So, $v'_1, v'_2, \dots, v'_\ell$ are orthonormal. Hence
\begin{align}
    \Pi' - \Pi^*&= \sum_{\ell =1}^r v'_\ell (v'_\ell)^\top - v_\ell v_\ell^\top \nonumber\\
    &= \sum_{\ell=1}^r -(2\eps - \eps^2) v_\ell v_\ell^\top + \sum_\ell \frac{(2\eps - \eps^2)}{\norm{u_\ell}_2^2}  u_\ell u_\ell^\top + \sum_\ell \frac{(1-\eps) \sqrt{2\eps - \eps^2}}{\norm{u_\ell}_2} \Big(u_\ell v_\ell^\top + v_\ell u_\ell^\top \Big) \label{eq:Pidiff}
\end{align}
Since each of the terms in \eqref{eq:Pidiff} are orthogonal (w.r.t. the trace inner product) we have
\begin{align}
    \norm{\Pi'-\Pi^*}_F^2 & = \sum_{\ell=1}^r (2\eps - \eps^2)^2 + \sum_{\ell=1}^r \frac{(2\eps - \eps^2)^2}{\norm{u_\ell}_2^4} + \sum_{\ell=1}^r 2 (2\eps - \eps^2) \cdot \frac{(1-\eps)^2}{\norm{u_\ell}_2^2} \nonumber\\
    &\ge r \eps = \Omega(\frac{\sqrt{r} \kappa \delta}{\sqrt{\lambda_1}}), \text{  with probability at least  } 1-n^{-\omega(1)}, \label{eq:lb:froblb}
\end{align}
for our choice of parameters (here we used \eqref{eq:lb:uproperties:1}). 
Then we lower bound the distance between $\Sigma$ and $\Sigma'$:
\begin{align*}
\Sigma'-\Sigma^* &= \sum_{\ell=1}^r \lambda_\ell \Big(v_\ell + \big(\frac{\sqrt{2\eps-\eps^2}}{(1-\eps)\norm{u_\ell}_2}\big) u_\ell \Big)\Big(v_\ell + \big(\frac{\sqrt{2\eps-\eps^2}}{(1-\eps)\norm{u_\ell}_2}\big) u_\ell \Big)^\top - \lambda_\ell v_{\ell} v_{\ell}^{\top}\\
&= \sum_{\ell=1}^r \lambda_{\ell} \frac{\sqrt{2\eps-\eps^2}}{(1-\eps)\norm{u_\ell}_2} (v_{\ell} u_{\ell}^{\top} + u_{\ell} v_{\ell}^{\top}) + \lambda_{\ell} \frac{2 \eps - \eps^2}{(1-\eps)^2 \|u_{\ell}\|_2^2} u_{\ell} u_{\ell}^{\top}.
\end{align*}
Because $v_{\ell}$ and $u_{\ell}$ are orthogonal and using \eqref{eq:lb:uproperties:1}, $\|\Sigma^*-\Sigma'\|_F^2$ is with high probability at least
\[
\big(\sum_{\ell=1}^r \lambda_{\ell}^2 \big)  \eps=\Big(\frac{\lambda_1^2+\dots+\lambda_r^2}{r}\Big)  \norm{\Pi'-\Pi^*}_F^2).
\]

We now show that $A'$ is a valid $\delta$-perturbation of $A$. Recall the definition of $A_j, A'_j$ in  \eqref{eq:constr:newA} respectively. For each fixed $j \in [m]$, by Lemma~\ref{lem:helper:perturb}, we have with probability at least $1-m^{-2}$ (over the randomness in $\set{\zeta^{(j)}_\ell: \ell \in [r]}$) that 
\begin{align*}
    \norm{A_j - A'_j}_\infty &= \Bignorm{ \sum_{\ell} \sqrt{\lambda_\ell} \zeta_\ell^{(j)} \cdot \frac{\sqrt{2\eps - \eps^2}}{(1-\eps)\norm{u_\ell}_2} u_\ell }_\infty\\
    &\le  \frac{2\sqrt{\log(r m)}}{(1-\eps)} \cdot \max_{\ell \in [r]} \sqrt{2 \eps \lambda_\ell} \frac{\norm{u_\ell}_\infty}{ \norm{u_\ell}_2},
\end{align*}
where the second term uses the fact that $u_1, \dots, u_r$ are disjoint and the concentration of Gaussian random variables (over $\zeta_{\ell}^{(j)}$). See also Lemma~\ref{lem:helper:perturb} for general $q$. After a union bound over all $j \in [m]$, and using our bounds on $\norm{u_\ell}_2$ and $\norm{u_\ell}_\infty$ from Lemma~\ref{lem:lb:uproperties} along with our definition of $\eps$,  we get with probability at least $1-\eta-\tfrac{1}{m}$,
\begin{align*} 
\max_{j \in [m]} \norm{A_j - A'_j}_\infty &\le \frac{2\sqrt{\log(r m)}}{(1-\eps)} \cdot \sqrt{2 \eps \lambda_1 }\cdot 2\sqrt{\frac{\log(r k' n)}{(k'-r)}} \cdot \frac{1}{1/2}\\
& = O\Big( \frac{\sqrt{\log (rm) \log n}}{\sqrt{k'}} \cdot \sqrt{\eps \lambda_1} \Big)  \le \delta,
\end{align*}    
since $\eps = c \delta^2 k'/ (\lambda_1\log(rm) \log n )$ for a small constant $c$ (and $\eps<1/2$).

\noindent {\em Upper bound on $\norm{\Pi'}_{\infty \to 2}$:} 
The proof will follow the same outline as for Theorem~\ref{thm:minmax_lower}. We compute the $\infty \to 1$ norm of $\Pi'$; recall that the $\infty \to 1$ norm satisfies the matrix norm monotone property (Lemma~\ref{lem:monotone}). From \eqref{eq:Pidiff},  triangle inequality and monotonicity,
\begin{align}
    \norm{\Pi'}_{\infty \to 1} & \le \underbrace{\norm{\Pi^*}_{\infty \to 1}}_{\text{equal to } \kappa^2} + \underbrace{2\Bignorm{\sum_{\ell} \eps u_\ell u_\ell^\top}_{\infty \to 1}}_{\text{bound using } \eqref{eq:helper:2}} + \underbrace{2 \Bignorm{\sum_{\ell} \sqrt{2\eps - \eps^2} u_\ell v_\ell^\top}_{\infty \to 1}}_{\text{bound using }\eqref{eq:helper:1}}. \label{eq:robustnessaltn:3}
\end{align}
We first bound the third term using \eqref{eq:helper:1} of Lemma~\ref{lem:lb:helper}. 
\begin{align*}
\Bignorm{\sum_\ell \sqrt{2\eps-\eps^2} u_\ell v_\ell^\top }_{\infty \to 1}& \le  \sqrt{2\eps} \sqrt{r} \norm{V}_{\infty \to 2}   \cdot \max_\ell \norm{u_\ell}_1  \le  \kappa  \sqrt{2 r k' \eps} \\
&\le \frac{\kappa^2}{ (\log n \log m)^{1/2}},
\end{align*}
by substituting the values for $k', \eps$ and using $r k' \eps =O( \kappa^2/(\log n \log m))$. 
Hence substituting in \eqref{eq:robustnessaltn:3} and using \eqref{eq:helper:2},
\begin{align*}
    \norm{\Pi'}_{\infty \to 1}  &\le \kappa^2+ 2 \eps \norm{U}_{2 \to 1}^2 + \kappa^2 \cdot \frac{1}{(\log n \log m)^{1/2}} \\
    &\le \kappa^2 + 2r  \eps \max_{\ell} \norm{u_\ell}_1^2 + \kappa^2 \cdot \Big( \frac{1}{(\log n \log m)^{1/2}} \Big) \\
&\le \kappa^2+ 4r \cdot \eps k'+o(\kappa^2) \le (1+o(1))\kappa^2. 
 \end{align*}

\end{proofof}

\subsubsection{Extension to general $\ell_q$ norm} \label{sec:lower:details}
Theorem~\ref{thm:inf_recover_lower} extends in a straightforward fashion to also hold for $\ell_q$ norms. 
\anote{Need to change the range of $\delta$}
\begin{theorem}\label{thm:inf_recover_lower_q}
Suppose we are given parameters $r\in \N, q \ge 1, \kappa \ge 2r^{1-2/q}$ and $\delta>0$. 
In the notation of Theorem~\ref{ithm:info_upper_gen}, for any $\Sigma^*$, given $m$ samples $A_1, \dots, A_m$ generated i.i.d. from $\calN(0,\Sigma^*)$ with 
$\kappa=\norm{\Pi^*}_{q \to 2}$ satisfying $\sqrt{r \lambda_1} (\kappa/n^{1-2/q}) \le \delta \le \sqrt{r \lambda_1}/\kappa$, 
there exists a covariance matrix $\Sigma'$ with a projector $\Pi'$ onto its top-$r$ principal subspace, and an alternate 
dataset $A'_1, \dots, A'_m$ drawn i.i.d. from $\calN(0,\Sigma')$ satisfying   $\norm{\Pi'}_{q \to 2} \le (1+o(1))\kappa$, and $\norm{A'_j - A_j}_q \le \delta ~\forall j \in [m]$,  
\begin{align*}
     \text{but }\norm{\Pi^*-\Pi'}_F^2 &\ge \big(\tfrac{\Omega(1)}{\sqrt{\lambda_1} \log(rm) \log n}\big) \cdot \sqrt{r }\kappa \delta,  ~\text{and } & \norm{\Sigma'_\bg - \Sigma_\bg}_F^2 \ge \tfrac{(\lambda_1^2+\dots+\lambda_r^2)}{r} \cdot \norm{\Pi'-\Pi^*}_F^2
\end{align*}
 In particular, when $\Sigma_\bg=(1+\theta) \Pi^*$ then $\Sigma'_\bg=(1+\theta') \Pi'$ with $\theta'=(1+o(1))\theta$. 

\end{theorem}

In fact the same construction holds using $u_1, \dots, u_r$ that are picked randomly but with disjoint support. However, there is a minor change in the parameters of the construction. We will set $\eps$ as before (and hence this will give the same lower bound on $\norm{\Pi'-\Pi^*}_F^2$ and $\norm{\Sigma'-\Sigma^*}_F^2$). We will set 
$$\eps=\frac{c \kappa \delta}{\sqrt{r \lambda_1} \log(rm) \log n} \text{    and    }(k')^{1-2/q}:= \Big( \frac{r^{2/q} \eps \lambda_\ell}{\delta^2 \log(rm) \log n}\Big),$$ 
for some constant $c>0$. The assumptions of the theorem ensure that $2r \le k' \ll n/r$ as required for the construction. 

We will need an additional simple claim that just extends Lemma~\ref{lem:lb:uproperties}.
\begin{lemma}\label{lem:lb:uproperties:new}
In the notation of Lemma~\ref{lem:lb:uproperties} for any $\eta<1$, with probability at least $(1 - \eta)$ we have
    \begin{align}
        \forall \ell \in [r], ~~~~~&  \norm{u_\ell}_q \le 3\sqrt{\log(r k'/\eta)} \dot (k')^{-1/2+1/q}.\label{eq:new:prop1}\\ 
         & \norm{u_\ell}_{q^*} \le 2 (k')^{1/2-1/q} .\label{eq:new:prop2}
    \end{align}
\end{lemma}
The proof follows directly from Lemma~\ref{lem:lb:uproperties} and using the relation between the $\ell_q, \ell_\infty$ norms, and $\ell_{q^*},\ell_1$ norms. 

\paragraph{Completing the proof of Theorem~\ref{thm:inf_recover_lower_q}.} The proof follows the same argument as the proof of Theorem~\ref{thm:inf_recover_lower}. As mentioned before, since we choose the same $\eps$, it suffices to argue about $\max_{j \in [m]} \norm{A_j - A'_j}_q$ and $\norm{\Pi'}_{q \to q^*}$.

To establish the upper bound on $\norm{\Pi'}_{q \to q^*}$ we use the bounds in Lemma~\ref{lem:lb:helper} and \eqref{eq:new:prop1}. We have from Lemma~\ref{lem:lb:helper}
\begin{align*}
    \norm{\Pi'}_{q \to q^*} & \le \norm{\Pi^*}_{\infty \to 1} + 2\eps \norm{UU^\top}_{q \to q^*} + 2\sqrt{2\eps - \eps^2} \norm{UV^\top}_{q \to q^*}\\
    &\le \kappa^2 + 2\eps r^{1-2/q} \Big(\max_{\ell \in [r]} \norm{u_\ell}_{q^*}\Big)^2+ 2 \sqrt{\eps} r^{1/2-1/q} \Big(\max_{\ell \in [r]} \norm{u_\ell}_{q^*}\Big) \cdot \norm{V}_{q \to 2} \\
    &\le \kappa^2 + 2\eps r^{1-2/q} (k')^{1-2/q}+ 2 \sqrt{\eps} r^{1/2-1/q} (k')^{1/1-1/q} \cdot \kappa \\
    &\le \kappa^2+ o(\kappa^2)+o(\kappa) \cdot \kappa = \kappa^2 (1+o(1)),  
\end{align*}
since from our choice of parameter $k'$, we have $\eps r^{1-2/q} \max_\ell \norm{u_\ell}_{q^*}^2= (\eps^2 r\lambda_1)/ (\delta^2 \log(rm) \log n) = o(\kappa^2)$. 

Finally, for the upper bound on $\max_{j \in [m]} \norm{A_j - A'_j}_q \le \delta$ we use Lemma~\ref{lem:helper:perturb} and \eqref{eq:new:prop2}. For each fixed $j \in [m]$, by Lemma~\ref{lem:helper:perturb}, we have with probability at least $1-m^{-2}$ (over the randomness in $\set{\zeta^{(j)}_\ell: \ell \in [r]}$) that 
\begin{align*}
    \norm{A_j - A'_j}_q &= \Bignorm{ \sum_{\ell} \sqrt{\lambda_\ell} \zeta_\ell^{(j)} \cdot \frac{\sqrt{2\eps - \eps^2}}{(1-\eps)\norm{u_\ell}_2} u_\ell }_q\\
    &\le  r^{1/q} \frac{2\sqrt{\log(r m)}}{(1-\eps)} \cdot \max_{\ell \in [r]} \sqrt{2 \eps \lambda_\ell} \frac{\norm{u_\ell}_q}{ \norm{u_\ell}_2}\\
    &\le r^{1/q} \cdot \sqrt{\log(rm)}(k')^{-1/2+1/q} \le \delta,
\end{align*}
for our choice of parameters and $k'$.  This establishes the statement of   Theorem~\ref{thm:inf_recover_lower_q} for general $q$.

\section{Statistical Upper bounds (computationally inefficient algorithm)}
We show the statistical upper bounds on the recovery of principle components in this section. By symmetrization (shown in Algorithm~\ref{algo:covariance}), we assume all data points are generated from $\calN(0,\Sigma^*)$ rather than $\calN(\mu,\Sigma^*)$ in this section.
\begin{theorem}\label{thm:inf_recover_upper_gen}
Given $q>2$, $n$, $r$, and $\kappa$, let $\mathcal{P}=\big\{\text{projection matrix } \Pi \big| \text{rank}=r \text{ and } \|\Pi\|_{q \to 2} \le \kappa \big\}$. Let $\Sigma$ be an unknown covariance matrix with eigenvalues $\lambda_1 \ge \lambda_2 \ge \cdots \ge \lambda_n$ whose projection matrix $\Pi^*$ of the top $r$ eigenspace is in $\mathcal{P}$. 

Let $\wt{A} \in \mathbb{R}^{n \times m}$ be the  $\delta$-perturbed (in $\ell_q$ norm) data matrix where each original column comes from $\calN(0, \Sigma^*)$ for any $\delta>0$, $\eps>0$ and $m \ge C \cdot \lambda_1^2 \cdot r^2 \kappa^2 \log n \cdot n^{2/q}/\eps^2$. Then
\[
\wt{\Pi} \overset{\text{def}}{=} \underset{\Pi \in \mathcal{P}}{\arg\min} \{ \|\wt{A}\|_F^2 - \|\Pi \wt{A} \|^2_F \} \]
satisfies $\|\wt{\Pi}^{\bot} \Pi^*\|_F^2 \le \frac{1}{\lambda_r - \lambda_{r+1}} \cdot O\left( \delta^2 \kappa^2 + \sqrt{\lambda_1 r} \cdot \delta \kappa + \eps \right)$ with probability 0.99.  Moreover, one can obtain $\wt{\Sigma}_{\bg}$ satisfying
$\| \wt{\Sigma}_{\bg}-\Sigma_{\bg} \|_F^2 \le O(\lambda_1^2 \cdot \|\wt{\Pi}^{\bot} \Pi^*\|_F^2 + \lambda_1 \kappa^2 \delta^2 + \kappa^4 \delta^4)$ where $\|\wt{\Pi}^{\bot} \Pi^*\|_F^2$ is upper bounded above.
\end{theorem}
\xnote{Feel free to revise this remark.}
\begin{remark}
Comparing to the computational upper bound in Theorem~\ref{thm:comp_upper_gen}, the main difference is the dependency of $m$ on $\kappa$: it becomes $\kappa^2$ here.
\end{remark}
We state the direct corollary in the spiked covariance model with $q=\infty$.
\begin{corollary}\label{cor:inf_recover_upper}
Given $n$, $r$, and $\kappa$, let $\mathcal{P}=\big\{\Pi \big| \text{rank}=r \text{ and } \|\Pi\|_{\infty \to 2} \le \kappa \big\}$. For any $\theta$ and $\Pi^* \in \mathcal{P}$, let $\wt{A} \in \mathbb{R}^{n \times m}$ be the  $\delta$-perturbed data matrix where each original column comes from $\calN(0, I + \theta \Pi^*)$. For any $\delta>0$, $\eps>0$ and $m \ge C \cdot (1+\theta)^2 \cdot r^2 \kappa^2 \log n/\eps^2$, 
\[
\wt{\Pi} \overset{\text{def}}{=} \underset{\Pi \in \mathcal{P}}{\arg\min} \{ \|\wt{A}\|_F^2 - \|\Pi \wt{A} \|^2_F \} \]
satisfies $\|\wt{\Pi}^{\bot} \Pi^*\|_F^2 \le \frac{1}{\theta} \cdot O\left( \delta^2 \kappa^2 + (1+\theta)^{1/2} \sqrt{r} \cdot \delta \kappa + \eps \right)$ with probability 0.99. 
\end{corollary}

We show two technical results to prove the main theorem. The first one bounds the deviation of the inner product between all projection matrices and the original data matrix (before perturbation), whose proof is defered to Section~\ref{sec:proof_lem_emp}.

\begin{lemma}\label{lem:emp_est_no_pert}
For any covariance matrix $\Sigma^*$ whose eigenvalues are at most $\lambda_{\max}$, let $A \in \mathbb{R}^{n \times m}$ be a data matrix where each column is generated from $\calN(0, \Sigma^*)$. 


Given $n$, $q$, $r$ and $\kappa$, let $\mathcal{P}=\big\{\Pi \big| \text{rank}=r \text{ and } \|\Pi\|_{q \to 2} \le \kappa \big\}$. Then for any $m \ge C \lambda_{\max}^2 \cdot \kappa^2 \log n \cdot n^{2/q}$ with a sufficiently large constant $C$, we have that with probability 0.99,
\[
\bigg| \Big\langle \frac{1}{m} A  A^{\top} - \Sigma^* , \Pi \Big\rangle \bigg| = r \cdot O\left( \frac{\lambda_{\max} \cdot \kappa \cdot \sqrt{\log n} \cdot n^{1/q}}{\sqrt{m}} \right) \text{ for all } \Pi \in \mathcal{P}.
\]
\end{lemma}

Then we bound the deviation of the inner product between all projection matrices and the actual data matrix (after perturbation) from the expectation.

\begin{claim}\label{clm:empirical_est_proj_gen}
Given $n$, $r$, and $\kappa$, let $\mathcal{P}=\big\{\Pi \big| \text{rank}=r \text{ and } \|\Pi\|_{q \to 2} \le \kappa \big\}$. For an unknown covariance matrix $\Sigma^*$, let $\lambda_{1}$ denote the largest eigenvalue of $\Sigma^*$.

Let $A \in \mathbb{R}^{n \times m}$ be the original data matrix where each column generated from $\calN(0, \Sigma^*)$ and $\wt{A}$ be its $\delta$-perturbation ($\ell_q$ norm in every column) for $m \ge C \lambda_{1}^2 \cdot \kappa^2 \log n \cdot n^{2/q}$ with a sufficiently large constant $C$. With probability 0.98,
\[
\bigg| \Big\langle \frac{1}{m} \wt{A} \cdot \wt{A}^{\top} - \Sigma^*, \Pi \Big\rangle \bigg| =O\left( \lambda_{1} \cdot r \kappa \cdot \sqrt{\frac{\log n}{m}} \cdot n^{1/q} + \delta^2 \kappa^2 + \sqrt{\lambda_{1} r} \cdot \delta \kappa \right) \text{ for all } \Pi \in \mathcal{P}.
\]
\end{claim}

\begin{proofof}{Claim~\ref{clm:empirical_est_proj_gen}}
We rewrite the left hand side as
\begin{align*}
& \bigg| \Big\langle \frac{1}{m} \wt{A}  \wt{A}^{\top} - \Sigma^* , \Pi \Big\rangle \bigg| \\
\le & \bigg| \Big\langle \frac{1}{m} A  A^{\top} - \Sigma^* + \frac{1}{m} (\wt{A} - A) A^{\top} + \frac{1}{m} \wt{A} (\wt{A} - A)^{\top}, \Pi \Big\rangle \bigg|\\
\le & \bigg| \Big\langle \frac{1}{m} A  A^{\top} - \Sigma^*, \Pi \Big\rangle \bigg| + \bigg| \Big\langle \frac{1}{m} (\wt{A} - A) A^{\top}, \Pi \Big\rangle \bigg| + \bigg| \Big\langle \frac{1}{m} \wt{A} (\wt{A} - A)^{\top}, \Pi \Big\rangle \bigg|\\
\le & \bigg| \Big\langle \frac{1}{m} A  A^{\top} - \Sigma^*, \Pi \Big\rangle \bigg| + 2\bigg| \Big\langle \frac{1}{m}  (\wt{A} - A) A^{\top}, \Pi \Big\rangle \bigg| + \bigg| \Big\langle \frac{1}{m} (\wt{A} - A) (\wt{A} - A)^{\top}, \Pi \Big\rangle \bigg|
\end{align*}
By Lemma~\ref{lem:emp_est_no_pert}, the first term $\bigg| \Big\langle \frac{1}{m} A A^{\top} - \Sigma^*, \Pi \Big\rangle \bigg| $ is upper bounded by $O\left( r \cdot \lambda_{1} \kappa \cdot \sqrt{\frac{\log n}{m}} \cdot n^{1/q} \right)$ with probability 0.99. Since $\|\wt{A}_i-A_i\|_{q} \le \delta$ and $\|\Pi\|_{q \rightarrow 2} \le \kappa$, the last term is upper bounded by
\[
\frac{1}{m} \bigg| \Big\langle (\wt{A} - A) (\wt{A} - A)^{\top}, \Pi^2 \Big\rangle \bigg| = \frac{1}{m} \|\Pi (\wt{A}-A)\|^2_F \le \delta^2 \kappa^2.
\]

We bound the second term here.
\[
 \frac{1}{m}  \bigg| \Big\langle (\wt{A} - A) A^{\top}, \Pi \Big\rangle \bigg|= \frac{1}{m} \bigg| \langle \Pi (\wt{A} - A), \Pi A \rangle \bigg| \le \frac{1}{m} \|\Pi(\wt{A}-A)\|_{F} \cdot \|\Pi A \|_F. 
\]
The first part $\|\Pi(\wt{A}-A)\|_{F}$ is always $\le \sqrt{m} \delta \kappa$ from the definition of $\Pi$. For the second part, notice that 
\[
\|\Pi A \|^2_F=\langle A A^{\top}, \Pi \rangle \le \langle m \Sigma^*, \Pi \rangle + \bigg| \langle A A^{\top} - m \Sigma^*, \Pi \rangle \bigg| \le \lambda_1 \cdot rm + O\left( r \lambda_1 \cdot \kappa \cdot \sqrt{m \log n} \cdot n^{1/q} \right),
\]
where the two bounds come from Fact~\ref{fact:eig_trace_ineq} and Lemma~\ref{lem:emp_est_no_pert} separately. So the second term is upper bounded by
\[
\frac{1}{m} \cdot \sqrt{m} \delta \kappa \cdot \left( \lambda_1 \cdot rm + C_0 \cdot r \lambda_1 \cdot  \kappa \cdot \sqrt{m \log n} \cdot n^{1/q} \right)^{1/2} \le \sqrt{r \lambda_1} \cdot \delta \kappa + \lambda_1^{1/2} \cdot  C_0^{1/2} \cdot \delta \kappa \cdot (\frac{r \kappa \sqrt{\log n} \cdot n^{1/q}}{\sqrt{m}})^{1/2}.
\]

So the total error is 
\begin{equation}\label{eq:stat_upper_total_error}
O\left( r \cdot \lambda_1 \kappa \cdot \sqrt{\frac{\log n}{m}} \cdot n^{1/q} \right) + \delta^2 \kappa^2 + \sqrt{\lambda_1 \cdot r} \delta \kappa + \lambda_1^{1/2} \cdot C_0^{1/2} \cdot \delta \kappa \cdot (\frac{r \kappa \sqrt{\log n} \cdot n^{1/q} }{\sqrt{m}})^{1/2}.
\end{equation}

Finally we simplify the error terms. The last term
\[
\lambda_1^{1/2} C_0^{1/2} \cdot \delta \kappa \cdot (\frac{r \kappa \sqrt{\log n} \cdot n^{1/q}}{\sqrt{m}})^{1/2} = O\left( \delta^2 \kappa^2 + \lambda_1 \cdot \frac{r \kappa \sqrt{\log n} \cdot n^{1/q}}{\sqrt{m}} \right),
\]
which are the first two terms in the total error \eqref{eq:stat_upper_total_error}.
\end{proofof}

Finally, we finish the proof of Theorem~\ref{thm:inf_recover_upper_gen}.

\begin{proofof}{Theorem~\ref{thm:inf_recover_upper_gen}}
Notice that the output projection $\wt{\Pi}$ could also be defined as $\underset{\Pi \in \mathcal{P}}{\arg\max} \{ \|\Pi \wt{A} \|^2_F \}$ and for any projection matrix $\Pi$,
\[
\frac{1}{m} \| \Pi \wt{A} \|_F^2=\frac{1}{m} \langle \wt{A} \wt{A}^{\top}, \Pi \rangle.
\]
By Claim~\ref{clm:empirical_est_proj_gen}, every $\Pi$ has $\frac{1}{m} \langle \wt{A} \wt{A}^{\top}, \Pi \rangle$ around $\langle \Sigma^*, \Pi \rangle \pm \Delta$ for 
\[
\Delta:=O\left( r \lambda_1 \cdot \kappa \cdot \sqrt{\frac{\log n}{m}} \cdot n^{1/q} + \delta^2 \kappa^2  + \sqrt{r \lambda_1} \cdot \delta \kappa \right) \text{ (the error in Claim~\ref{clm:empirical_est_proj_gen}).}
\]  Since $\wt{\Pi}$ attains a better objective value than $\Pi^*$, we have
\begin{align*}
\langle \Sigma^*, \wt{\Pi} \rangle & \ge \Big\langle \frac{1}{m} \wt{A} \wt{A}^{\top}, \wt{\Pi} \Big\rangle -  \Delta\\
& \ge \Big\langle \frac{1}{m} \wt{A} \wt{A}^{\top}, \Pi^* \Big\rangle - \Delta \tag{using the definition of $\wt{\Pi}$} \\
& \ge \langle \Sigma^*, \Pi^* \rangle -2\Delta. 
\end{align*}
Next, we apply Claim~\ref{clm:distance_X_proj} to conclude $\langle \Pi^*, \wt{\Pi} \rangle \ge r - \frac{2\Delta}{\lambda_r - \lambda_{r+1}}$, which upper bounds $\|\wt{\Pi}^{\bot} \Pi^*\|_F^2 \le \frac{2\Delta}{\lambda_r - \lambda_{r+1}}$. Finally we use Theorem~\ref{thm:recover_top_covariance} to get $\wt{\Sigma}_{\bg}$ satisfying $\|\wt{\Sigma}_{\bg} - \Sigma_{\bg}\|_F^2 \le O(\lambda_1^2 \cdot \frac{2 \Delta}{\lambda_r - \lambda_{r+1}} + \lambda_1 \kappa^2 \delta^2 + \kappa^4 \delta^4)$.
\end{proofof}


\subsection{Proof of Lemma~\ref{lem:emp_est_no_pert}}\label{sec:proof_lem_emp}
We use the following concentration result from~\cite{Mendelson2010} to bound the supremum. 
\begin{lemma}[See Corollary 4.1 in \cite{VuLei12}]\label{lem:guassian_non_id_covariance}
Let $A_1,\ldots,A_m \in \mathbb{R}^n$ be i.i.d. mean 0 random vectors with \[ \Sigma=\E A_1 A_1^{\top} \text{ and } \sigma=\sup_{\|u\|_2=1} \big\| \langle A_1, u \rangle \big\|_{\psi_2}.\]
For $S_n=\frac{1}{m} \sum_{i=1}^m A_i \cdot A_i^{\top}$ and a symmetric subset $\mathcal{V}$ in $\mathbb{R}^{n}$, we have
\[
\E_{A_1,\ldots,A_m}\left[ \sup_{v \in \mathcal{V}} \bigg| \big\langle S_n - \Sigma, v v^\top \big\rangle \bigg| \right] \le c \left( \frac{\sigma^2}{\sqrt{m}} \cdot \underset{v \in \mathcal{V}}{\sup} \|v\|_2 \cdot \E_g \bigg[\sup_{v \in \mathcal{V}} \langle g, v \rangle \bigg] + \frac{\sigma^2}{m} \E_g \bigg[\sup_{v \in \mathcal{V}} \langle g, v \rangle \bigg]^2\right)
\]
for a vector $g \in \mathbb{R}^{n}$ with i.i.d. Gaussian entries and a universal constant $c$.
\end{lemma}

To use the above lemma, we first upper bound $\sigma^2$ in our setting.
\begin{claim}
Let $X \sim \calN(0, \Sigma^*)$ for a matrix $\Sigma^*$ with eigenvalues at most $\lambda_{\max}$. Then $\big\| \langle X, u \rangle \big\|_{\psi_2} \le \sqrt{\lambda_{\max}(\Sigma^*)}$ for any $u$ with $\|u\|_2=1$.
\end{claim}
\begin{proof}
Let $v_1,\ldots,v_n$ be the eigenvectors of $\Sigma^*$ with eigenvalues $\lambda_1,\ldots,\lambda_n$. Then $\langle X, u \rangle = \sqrt{\lambda_1} \cdot \langle v_1,u \rangle g_1 + \cdots + \sqrt{\lambda_n} \cdot \langle v_n,u \rangle g_n$ for i.i.d. Gaussian random variable $g_1,\ldots,g_n$. So the variance is $\lambda_1 \langle v_1,u \rangle^2 + \cdots + \lambda_n \langle v_n, u \rangle^2 \le \max\{\lambda_1,\ldots,\lambda_n\}$ and 
\[
\big\| \langle X, u \rangle \big\|_{\psi_2} \le \sqrt{\lambda_{\max}}.
\]
\end{proof}

We apply Lemma~\ref{lem:guassian_non_id_covariance} to all vectors that could be in the basis of possible $\Pi$.

\begin{claim}\label{clm:non_id_Gaussian_sparse_vectors}
For any covariance matrix $\Sigma^*$ with eigenvalues at most $\lambda_{\max}$, let $A_1,\ldots,A_m \in \mathbb{R}^n$ be i.i.d. vectors generated from $\calN(0, \Sigma^*)$. Given $n$ and $q$, let $\mathcal{V}$ be the set of all vectors $v$ with $\|v\|_2 = 1$ and $\|v\|_{q^*} \le \kappa$. 

Then for any $m \ge C \lambda_{\max}^2 \cdot \kappa^2 \log n \cdot n^{2/q}$ with a sufficiently large constant $C$, we have that with probability 0.99,
\[
\bigg| \Big\langle \frac{1}{m} \sum_{i=1}^m A_i  A_i^{\top} - \Sigma^* , v v^{\top} \Big\rangle \bigg| = O\left( \frac{\lambda_{\max} \kappa \sqrt{\log n} \cdot n^{1/q}}{\sqrt{m}} \right) \text{ for all } v \in \mathcal{V}.
\]
\end{claim}
\begin{proof}
To apply Lemma~\ref{lem:guassian_non_id_covariance}, we notice that 
$\sup_{v \in \mathcal{V}} \|v\|_2 = 1$
and
\[
\E_g \left[ \sup_{v \in \mathcal{V}} \langle g, v \rangle \right] \le \E \left[ \sup_v \|g\|_{q} \cdot \|v\|_{q^*} \right] = \E[\|g\|_{q}] \cdot \sup \|v\|_{q^*}= O(n^{1/q} \sqrt{\log n} \cdot \kappa). 
\]
Thus Lemma~\ref{lem:guassian_non_id_covariance} shows that for some absolute constant $c'>0$
\[
\E_{A_1,\ldots,A_m} \left[ \sup_{v \in \mathcal{V}} \bigg| \Big\langle \frac{1}{m} \sum_{i=1}^m A_i  A_i^{\top} - \Sigma^*, v v^{\top} \Big\rangle \bigg| \right] =  \frac{c'\lambda_{\max} \cdot 1 \cdot \kappa \sqrt{\log n} \cdot n^{1/q}}{\sqrt{m}} + \frac{c' \lambda_{\max} \kappa^2 \log n \cdot n^{2/q}}{m}.
\]
When $m>C \lambda_{\max}^2 \cdot \kappa^2 \log n \cdot n^{2/q}$, the right hand is at most twice the first term $O(\frac{\lambda_{\max} \cdot  \kappa \sqrt{\log n} \cdot n^{1/q}}{\sqrt{m}})$. Next we apply the Markov inequality to replace the expectation by probability 0.99.
\end{proof}
Lemma~\ref{lem:emp_est_no_pert} follows as a corollary of the above claim: for any $\Pi$ of rank $r$ and $\|\Pi\|_{q \rightarrow 2} \le \kappa$, we have $\|\Pi\|_{2 \rightarrow q^*}=\|\Pi\|_{q \rightarrow 2}=\kappa$ such that all its eigenvectors $v_1,\ldots,v_r$ are in $\mathcal{V}$ with $\|v_i\|_{q^*} \le \kappa$ (by considering $\|\Pi v_i\|_{q^*} \le \kappa$). Thus
\begin{align*}
    \bigg| \Big\langle \frac{1}{m} \sum_{i=1}^m A_i  A_i^{\top} - \Sigma^* , \Pi \Big\rangle \bigg| & = \bigg| \Big\langle \frac{1}{m} \sum_{i=1}^m A_i  A_i^{\top} - \Sigma^* , \sum_{j=1}^r v_j v_j^{\top} \Big\rangle \bigg| \\
    & \le \sum_{j=1}^r \bigg| \Big\langle \frac{1}{m} \sum_{i=1}^m A_i  A_i^{\top} - \Sigma^* ,  v_j v_j^{\top} \Big\rangle \bigg| = r \cdot O\left( \frac{\lambda_{\max} \kappa \cdot \sqrt{\log n} \cdot n^{1/q}}{\sqrt{m}} \right).
\end{align*}

\section{Robust Mean Estimation}
\label{sec:mean_comp}

In this section we present an analysis of the robust mean estimation procedure sketched below, thereby establishing Proposition~\ref{prop:intro:robust_mean}. 
\begin{algorithm}[H]
\caption{Mean Estimation under Adversarial Perturbations}\label{algo:mean}

\begin{algorithmic}[1]
\Function{AdvRobustMean}{$m$ samples $\tilde{A}_1, \dots, \tilde{A}_{m} \in \R^{n}$, norm $q$, perturbation $\delta$, error $\eta$}
\State Compute the empirical mean $\mu'$ of all the given samples.
\State Output $\tilde{\mu}$, where $\tilde{\mu}$ is the point in the $\ell_q$ ball of size $\delta+\eta$ around $\mu'$ with the minimum $\ell_{q^*}$ norm i.e.,
$$ \min_{u \in \R^n} \norm{u}_{q^*}^{q^*} , \text{ s.t. } \norm{u-\mu'}_q \le \delta +\eta.$$
\EndFunction
\end{algorithmic}
\end{algorithm}

We remark that the above algorithm in the case of $q=\infty$ specializes to 
$\forall i \in [n],~~\tilde{\mu}(i)=\text{sign}(\mu'(i))\cdot \max\set{ |\mu'(i)| - (\delta+\eta),0}$.
This is the same as the soft-thresholding algorithm that has been explored in the sparse mean estimation literature. 
More generally, we will prove the statement for any $\ell_q$ norm for $q \geq 2$. The main theorem of this section is the following

\begin{proposition}
\label{thm:robust_mean}
Fix $q \geq 2$. Suppose we have $m$ samples drawn according to the Adversarial Perturbation model with $\ell_q$ perturbations.  There is a polynomial time algorithm (Algorithm~\ref{algo:mean}) that outputs an estimate $\hat{\mu}$ for the (unknown) mean $\mu$ such that with probability at least $(1-1/n)$, 
\begin{align}
        \norm{\hat{\mu} - \mu}_2^2 &\leq 4\min\Big\{\norm{\mu}_{q^*} ( \delta+ \eta),  n^{1-\frac{1}{q}} ( \delta+ \eta )^2 \Big\}, \text{ where } \eta:=2\sigma n^{\frac{1}{q}}\sqrt{\frac{\log n}{m}}.
\end{align}
\end{proposition}
\begin{proof}
Let $\mu' = \text{mean}(\tilde{A})$. Since $\|\tilde{A}_j - A_j\|_q \leq \delta$ for each $j \in [m]$, we know that $\norm{\mu' - \text{mean}(A)}_q \leq \delta$. Furthermore, from standard Gaussian concentration as stated in Fact~\ref{fact:gaussian-qnorm-concentration} below we have that with probability at least $1-\frac{1}{n}$ it holds that
\begin{align}
    \label{eq:mean-q-norm-concentration}
    \norm{\mu - \text{mean}(A)}_q &\leq \eta = 2\sigma n^{\frac{1}{q}}\sqrt{\frac{\log n}{m}}.
\end{align}
This implies that with probability at least $1-\frac{1}{n}$, 
\begin{align}
    \label{eq:mean-valid-sol}
\norm{\mu - \mu'}_q \leq \delta + \eta
\end{align}
and hence is a valid solution to the convex program in Algorithm~\ref{algo:mean}. Moreover the convex program can be solved in polynomial time using the Ellipsoid method. This is because the objective is separable over the data points, and for each constraint is of the form $\norm{z}_p \le \tau$, where $\tau$ is specified and $p \ge 1$. A simple hyperplane separation oracle for a constraint of the form $\norm{z}_p \le \tau$ is given by the duality since 
$$\norm{z}_p = \max_{y \in \R^n: \norm{y}_{p^*} \le 1} \iprod{y,z} = \Big\langle \frac{z^*}{\norm{z^*}_{p^*}},z \Big \rangle, \text{ where } z^*_i= \text{sign}(z_i) |z(i)|^{p-1}~~ \forall i \in [n].$$
Hence a hyperplane of the form $\iprod{w,z} \le \tau$ with $w = z^*/\norm{z^*}_{p^*}$ gives a valid separation oracle. A similar separation oracle can also be used for the objective. 
(Note that one can also use the projected sub-gradient method for a more effective algorithm). 

This implies that the Algorithm outputs a vector $\hat{\mu}$ in polynomial time. It satisfies
\begin{align}
    \label{eq:mean-qstar-bound}
    \norm{\hat{\mu}}_{q^*} &\leq \norm{{\mu}}_{q^*}
\end{align}
Hence, via H\"older's inequality we get that
\begin{align}
\label{eq:mean-error-bound-1}
    \|\hat{\mu} - \mu\|^2_2 &\leq \norm{\hat{\mu} - \mu}_q \norm{\hat{\mu} - \mu}_{q^*} \nonumber \\
    &\leq (\norm{\hat{\mu} - \mu'}_q + \norm{{\mu} - \mu'}_q)(\norm{\hat{\mu}}_{q^*} + \norm{\mu}_{q^*}) \nonumber \\
    &\leq 2(\norm{\hat{\mu} - \mu'}_q + \norm{{\mu} - \mu'}_q)\norm{\mu}_{q^*} \,\, \text{[from \eqref{eq:mean-qstar-bound}]} \nonumber \\
    &\leq 4\norm{\mu - \mu'}_q \norm{\mu}_{q^*} \,\, \text{[from the optimality of $\hat{\mu}$.]} \nonumber \\
    &\leq 4(\delta + \eta)\norm{\mu}_{q^*} \,\, \text{[from \eqref{eq:mean-valid-sol}]}
\end{align}
Alternately, using the fact that for any vector $x \in \mathbb{R}^n$, $\norm{x}_p \leq n^{\frac{1}{p} - \frac{1}{q}} \norm{x}_q$ we get that
\begin{align}
\label{eq:mean-error-bound-2}
    \|\hat{\mu} - \mu\|^2_2 &\leq n^{1-\frac{1}{q}}\norm{\hat{\mu} - \mu}^2_q \nonumber \\
    &\leq n^{1-\frac{1}{q}} \Big( \norm{\hat{\mu} - \mu'}_q + \norm{{\mu} - \mu'}_q\Big)^2 \nonumber \\
    &\leq 4n^{1-\frac{1}{q}} \norm{\mu - \mu'}^2_q \,\, \text{[from the optimality of $\hat{\mu}$.]} \nonumber \\
    &\leq 4n^{1-\frac{1}{q}} (\delta + \eta)^2 \,\, \text{[from \eqref{eq:mean-valid-sol}]}.
\end{align}
Combining \eqref{eq:mean-error-bound-1} and \eqref{eq:mean-error-bound-2} we get the claim. Setting $q = \infty$ establishes Proposition~\ref{prop:intro:robust_mean} from the introduction.
\end{proof}
To complete the argument we provide a self contained proof of the fact stated below.
\begin{fact}
\label{fact:gaussian-qnorm-concentration}
Fix $q \geq 2$. Let $A_1, \dots, A_m$ be drawn i.i.d. from $N(0, \Sigma_{n \times n})$ with $\norm{\Sigma} \leq \sigma^2$. Then with probability at least $1-\frac{1}{n}$ it holds that,
\begin{align*}
    \norm{\frac{1}{m} \sum_{i=1}^m A_i}_q &\leq 2\sigma n^{\frac{1}{q}} \sqrt{\frac{\log n}{m}}.
\end{align*}
\begin{proof}
Noticing that each coordinate of $\frac{1}{m}\sum_{i=1}^m A_i$ is a mean Gaussian with variance bounded by $\sigma^2/m$ and using union bound we get that with probability at least $1-\frac{1}{n}$,
\begin{align*}
    \norm{\frac{1}{m} \sum_{i=1}^m A_i}_{\infty} &\leq 2\sigma \sqrt{\frac{\log n}{m}}.
\end{align*}
Then it easily follows that with probability at least $1-\frac{1}{n}$,
\begin{align*}
    \norm{\frac{1}{m} \sum_{i=1}^m A_i}_{q} &\leq n^{\frac{1}{q}}\norm{\frac{1}{m} \sum_{i=1}^m A_i}_{\infty}\\
    &\leq 2\sigma n^{\frac{1}{q}} \sqrt{\frac{\log n}{m}}.
\end{align*}
\end{proof}
\end{fact}    
Notice that the bound of $n^{1-\frac{1}{q}}(\delta + \eta)^2$ is the naive bound that is simply achieved by always outputting the mean of the points in $\tilde{A}$. Hence, for small values of the perturbation $\delta$, the algorithm achieves a non-trivial guarantee of $\|\mu\|_{q^*} (\delta + \eta)$. In fact we next show that the guarantee of the algorithm is optimal. In particular, provide an instance wise lower bound, stated below, for robust mean estimation in our model of corruption.
\begin{proposition}
\label{thm:state_lower_mean}
Fix $q = \infty$. Let $\mu$ be any vector such that the analytical sparsity of $\mu$, i.e., $\frac{\norm{\mu}_1}{\norm{\mu}}$ is bounded by $\sqrt{n}/4$. Then there exist $\delta, \sigma > 0$ and another vector $\norm{\mu'}$ such that $\frac{\norm{\mu'}_1}{\norm{\mu'}_2} = \frac{\norm{\mu}_1}{\norm{\mu}_2} (1+o(1))$, and $\norm{\mu-\mu'}_2 = \Omega(\sqrt{\delta \norm{\mu}}_1)$ and with high probability, i.i.d. samples $A_1, A_2, \dots A_m$ generated from $\mathcal{N}(\mu, \sigma^2 I)$ and $\tilde{A}_1, \tilde{A}_2, \dots \tilde{A}_m$ generated from $\mathcal{N}(\mu', \sigma^2 I)$ satisfy $\norm{A_j - \tilde{A}_j}_\infty \leq \delta$, for all $j \in [m]$.
\end{proposition}
\begin{proof}
The construction builds upon the argument presented in~\cite{ACCV} with most of the details unchanged. We provide a proof sketch here. Pick a subset $S$ of $s = (\frac{\norm{\mu}_1}{\norm{\mu}_2})^2$ coordinates and define $\mu' = \mu + \delta sign(\mu_S)$, where $\mu_S$ is the vector that equals $\mu$ over $S$ and $0$ outside of $S$. Notice that since the analytical sparsity of $\mu$ is bounded by $\sqrt{n}/4$, $S$ will be non-empty. We will pick $\delta$ such that $\delta = o(\norm{\mu}^2)/\norm{\mu}_1$. It is easy to see that $\norm{\mu'}^2 \geq \norm{\mu}^2$ and we also have that $\norm{\mu}_1 = \norm{\mu}_1 + \delta s = \frac{\norm{\mu}_1}{\norm{\mu}_2} (1+o(1))$. Also if $\sigma$ is small enough then samples generated from $\mathcal{N}(\mu, \sigma^2 I)$ and from $\mathcal{N}(\mu', \sigma^2 I)$ will be $\delta$-close to each other. Finally, notice that 
\begin{align*}
    \|\mu - \mu'\| &= \delta\sqrt{s}\\
    &= \Omega(\sqrt{\delta \norm{\mu}_1}).
\end{align*}
\end{proof}

\end{document}